\documentclass{article}

\usepackage{microtype}
\usepackage{graphicx}
\usepackage{subfigure}
\usepackage{booktabs} % for professional tables

\usepackage{enumerate}
\usepackage{amsmath,amssymb,amsthm}
\usepackage{xcolor}
\usepackage[textsize=small]{todonotes}
\usepackage[normalem]{ulem}
\usepackage{enumerate, enumitem}
\usepackage{mathtools}
\usepackage{graphicx}
\usepackage{multirow}
\graphicspath{{figures/}}
\usepackage{dsfont}

\newcommand{\eg}{{\it e.g., }}
\newcommand{\ie}{{\it i.e., }}

\newcommand{\reals}{\mathbb{R}}

\newcommand{\Ex}[1]{\mathbb{E}{\left[{#1}\right]}}
\renewcommand{\Pr}[1]{\mathbb{P}{\left({#1}\right)}}
\newcommand{\eqdist}{\mathrel{\stackrel{d}{=}}}

\newcommand{\argmin}{\mathop{\rm argmin}}

\newcommand{\argmax}{\mathop{\rm argmax}}

\newcommand{\matL}{\mathcal{L}}
\newcommand{\hatL}{\widehat{\mathcal{L}}}

\theoremstyle{definition}
\newtheorem{definition}{Definition}[section]
\newtheorem{theorem}{Theorem}

\newtheorem{lemma}{Lemma}
\newtheorem{proposition}{Proposition}

\newtheorem{remark}{Remark}

\newtheorem{innercustomgeneric}{\customgenericname}
\providecommand{\customgenericname}{}
\newcommand{\newcustomtheorem}[2]{%
  \newenvironment{#1}[1]
  {%
   \renewcommand\customgenericname{#2}%
   \renewcommand\theinnercustomgeneric{##1}%
   \innercustomgeneric
  }
  {\endinnercustomgeneric}
}
\newcustomtheorem{customthm}{Theorem}
\newcustomtheorem{customprop}{Proposition}
\newcustomtheorem{customlemma}{Lemma}
\newcustomtheorem{customcorollary}{Corollary}
\newcustomtheorem{customdef}{Definition}
\newcustomtheorem{customprob}{Problem}

\newif\iftodos

\newcommand{\abs}[1]{\lvert{#1}\rvert}

\newcommand{\indicator}[1]{\mathds{1}\{{#1}\}}

\DeclarePairedDelimiterX{\infdivx}[2]{(}{)}{%
  #1\;\delimsize\|\;#2%
}

\newcommand{\normal}{{\mathcal N}}

\todosfalse

\newcommand{\betaML}{\hat{\beta}_\mathrm{ML}}

\newcommand{\dkl}[2]{D_{\mathrm{KL}}\left({#1}\middle\|{#2}\right)}

\newcommand{\Ephat}[1]{\mathbb{E}_{\hat{p}}\left[{#1}\right]}
\newcommand{\Ept}[1]{\mathbb{E}_{p}\left[{#1}\right]}
\newcommand{\Ep}[1]{\mathbb{E}_{p(\beta)}\left[{#1}\right]}

\newcommand{\E}{\mathbb{E}}
\newcommand{\bfx}{\mathbf{x}}

\newcommand{\gammatilde}{\tilde{\gamma}}

\usepackage{hyperref}

\usepackage[accepted]{icml2021}

\icmltitlerunning{A Wasserstein Minimax Framework for Mixed Linear Regression}

\begin{document}

\twocolumn[
\icmltitle{A Wasserstein Minimax Framework for Mixed Linear Regression}



\icmlsetsymbol{alphabetic}{*}

\begin{icmlauthorlist}
\icmlauthor{Theo Diamandis}{alphabetic,MIT}
\icmlauthor{Yonina C. Eldar}{Weizmann}
\icmlauthor{Alireza Fallah}{MIT}
\icmlauthor{Farzan Farnia}{MIT}
\icmlauthor{Asuman Ozdaglar}{MIT}
\end{icmlauthorlist}

\icmlaffiliation{MIT}{Department of Electrical Engineering \& Computer Science, MIT, USA}
\icmlaffiliation{Weizmann}{Faculty of Math and Computer Science, Weizmann Institute of Science, Israel}

\icmlcorrespondingauthor{Theo Diamandis}{tdiamand@mit.edu}
\icmlcorrespondingauthor{Alireza Fallah}{afallah@mit.edu}
\icmlcorrespondingauthor{Farzan Farnia}{farnia@mit.edu}

\icmlkeywords{Machine Learning, ICML}

\vskip 0.3in
]



\printAffiliationsAndNotice{\textsuperscript{*}The authors are in alphabetical order.}

\begin{abstract}
Multi-modal distributions are commonly used to model clustered data in statistical learning tasks. In this paper, we consider the Mixed Linear Regression (MLR) problem. We propose an optimal transport-based framework for MLR problems, Wasserstein Mixed Linear Regression (WMLR), which minimizes the Wasserstein distance between the learned and target mixture regression models. Through a model-based duality analysis, WMLR reduces the underlying MLR task to a nonconvex-concave minimax optimization problem, which can be provably solved to find a minimax stationary point by the Gradient Descent Ascent (GDA) algorithm. In the special case of mixtures of two linear regression models, we show that WMLR enjoys global convergence and generalization guarantees. We prove that WMLR’s sample complexity grows linearly with the dimension of data. Finally, we discuss the application of WMLR to the federated learning task where the training samples are collected by multiple agents in a network. Unlike the Expectation Maximization algorithm, WMLR directly extends to the distributed, federated learning setting. We support our theoretical results through several numerical experiments, which highlight our framework’s ability to handle the federated learning setting with mixture models.
\end{abstract}

\section{Introduction}

Learning mixture models which describe data collected from multiple subpopulations has been a basic task in the machine learning literature. Multi-modal distributions typically emerge in distributed learning settings where the training data are gathered from a heterogeneous group of users.
For example, speech data or genetic data may exhibit a clustered distribution based on language and ethnicity, respectively.
Such settings require learning methods that can efficiently learn an underlying multi-modal distribution in both a centralized and a distributed setting.

In this paper, we specifically focus on Mixed Linear Regression (MLR) problems. In the MLR problem, the output variable for every user is a randomized linear function of the feature variables, generated according to one of $k$ unknown linear regression models. This structured model provides a simple but expressive framework to analyze multimodal labeled data. The clustered structure of MLR appears in several supervised learning applications. For example, users of a recommendation engine usually have unknown yet clustered sets of preferences which leads to multiple regression models. In genetic datasets, the underlying cell-type of collected samples is a latent variable that can result in different linear regression models. Under such scenarios, the cluster identity is an unknown latent variable that should be estimated along with the linear regression models.

To address the MLR problem, we propose an optimal transport-based learning framework, which we refer to as \emph{Wasserstein Mixed Linear Regression (WMLR)}. We revisit optimal transport theory to formulate the centralized MLR task as a minimax optimization problem solved by the WMLR algorithm. The formulated minimax problem is the dual problem of minimizing the Wasserstein distance between the target and learned mixture regression models. Because the original minimax problem formulated by applying the standard Kantorovich duality \cite{villani2008optimal} incurs significant computational and statistical costs, we reduce the minimax learning task to a tractable problem by a model-based simplification of the dual maximization variables. 

For a general MLR problem, we prove that the proposed minimax problem can be reduced to a nonconvex-concave optimization problem for which the gradient descent ascent (GDA) algorithm is guaranteed to converge to a stationary minimax solution. Furthermore, under the well-studied benchmark of a mixture of two symmetric linear regression models, we theoretically support our framework by providing global convergence and generalization guarantees. In particular, we show that our framework can provably converge to the global minimax solution and properly generalize from the empirical distribution of training samples to the underlying mixture regression model.

Next, we examine the WMLR algorithm for MLR tasks in the distributed federated learning setting \cite{mcmahan17a}. In a federated learning task, a set of local users connected to a central server train a global model over the samples observed in the network. While the Expectation-Maximization (EM) algorithm is widely considered as the state-of-the-art approach for centralized MLR problems, in the federated learning setting, the maximization step of every iteration of the EM algorithm requires multiple gradient computation and communication steps to obtain an exact solution via an iterative method. As a result, the EM algorithm cannot be decomposed into an efficient distributed form.

On the other hand, we show that while the maximization step in the EM algorithm does not directly reduce to a distributed form, the gradient steps of WMLR extend to the federated learning setting. As a result, our theoretical guarantees in the centralized case also hold in the federated learning setting. Finally, we present the results of several numerical experiments which support the flexibility of our proposed minimax framework in both centralized and decentralized learning tasks.

Our main contributions are summarized as follows:
\begin{enumerate}
    \item We propose a minimax framework, Wasserstein Mixed Linear Regression (WMLR), to solve the MLR problem using optimal transport theory.
    \item We reduce WMLR to a tractable nonconvex-concave minimax optimization problem, which can be solved by the GDA algorithm.
    \item We show that WMLR enjoys convergence and generalization guarantees in both centralized and federated learning settings in the symmetric MLR case.
    \item We support WMLR's theoretical guarantees with numerical experiments for the centralized and federated learning settings.
\end{enumerate}

\subsection{Related Work}
The MLR model, introduced in the statistics literature by \citet{de1989mixtures} and later in the machine learning literature by \citet{jordan1994hierarchical} as ``hierarchical mixtures of experts", provides a simple but expressive framework to analyze multimodal data. However, despite the simplicity of the model, learning mixed regression models is computationally difficult; the maximum likelihood problem is intractable in the general case \cite{yi2014alternating}.

\paragraph{EM-based Algorithms for MLR} \citet{kwon2019global} prove global convergence for balanced mixtures of symmetric two component linear regressions. Several other papers have extended \cite{kwon2019global}'s results to unequally weighted components and $K$ components in the noiseless setting (See \cite{kwon2020converges} and references therein). Furthermore, \citet{kwon2020converges} prove local convergence for $k$-MLR in the noisy case. However, the EM algorithm still requires ``good" initialization for convergence to the optimal solution  \cite{balakrishnan2017statistical}. For finding such a good initialization, several methods have been proposed in the EM literature, including methods based on PCA \cite{yi2014alternating} and method of moments \cite{chaganty2013spectral}. Without proper initialization, the EM algorithm has been empirically shown to find poor estimations due to EM's ``sharp" selection of clusters.

\paragraph{Gradient-based Algorithms for MLR} The traditional EM algorithm fully solves a maximization at each step, resulting in the ``sharp" behavior. Several alternative algorithms have been proposed that take a gradient descent approach. First-order EM, where only one gradient step in the maximization problem is taken, enjoys a local convergence guarantee \cite{balakrishnan2017statistical}. \citet{zhong2016mixed} show local convergence for a nonconvex objective function that solves the $k$-MLR problem. \citet{chen2014convex} provide a convex formulation for the two component case, but it is unclear how this method generalizes to $k > 2$.

\paragraph{Federated Learning with Heterogeneous Data}
Several approaches have been proposed in the literature to deal with heterogeneity in FL, including correcting the local updates \cite{scaffold2020} or using meta-learning techniques for achieving personalization \cite{NEURIPS2020_24389bfe}.  
In particular, clustering is one of these approaches where the idea is to group client population into clusters
\cite{sattler2020clustered,ghosh2020efficient,mansour2020three,li2021fedbn}. Most relevant to our work, \citet{mansour2020three} and \citet{ghosh2020efficient} propose alternating minimization algorithms, where at each step the agents find their cluster identity, compute the loss function gradient, and send them back to the server. \citet{ghosh2020efficient} further prove convergence guarantees for linear models and strongly convex loss functions under certain initialization assumptions. These frameworks include a much larger class of problems than MLR, but they do not enjoy the same global convergence and optimality guarantees that WMLR has for the MLR case.

\paragraph{Minimax Frameworks for Federated Learning} Several related works explore the applications of minimax frameworks for improving the fairness and robustness of federated learning algorithms. \citet{mohri2019agnostic} introduce Agnostic Federated Learning as a min-max framework that improves the fairness properties in federated learning tasks. \citet{reisizadeh2020robust} propose a minimax federated learning framework that is robust to affine distribution shifts. Similarly, \citet{deng2021distributionally} develop a distributionally-robust federated learning algorithm using a minimax formulation. However, unlike our work the mentioned frameworks do not address the clustered federated learning problem. 

\paragraph{Generative Adversarial Networks (GANs)}
Similar to our proposed framework, GANs \cite{goodfellow2014generative} reduce the distribution learning problem to a minimax optimization task. Optimal transport costs have been similarly used to formulate GAN problems \cite{arjovsky2017wasserstein,sanjabi2018convergence,farnia2018convex,feizi2020understanding}. Also, \citet{genevay2018learning} formulate a min-max problem for learning generative models using the optimal transport-based Sinkhorn loss functions. On the other hand, since standard GAN formulations perform suboptimally in learning multimodal distributions \cite{goodfellow2016nips}, \citet{farnia2020gat} propose a similar model-based minimax approach to successfully learn mixtures of Gaussians. \citet{mena2020sinkhorn} introduce the optimal transport-based Sinkhorn EM framework for learning mixture models.  However, while the mentioned minimax frameworks focus on unsupervised learning tasks, our proposed approach addresses the supervised MLR problem.

%%%%%%
\subsection{Notation}
For two random variables $Y$ and $Y'$, $Y \eqdist Y'$ means that $Y$ and $Y'$ have the same distribution. For a finite set $\mathcal{A}$, $\text{Unif}(\{\mathcal{A}\})$ stands for the uniform distribution over $\mathcal{A}$, and $I_A(u)$ is the indicator function of $A$, \ie $I_A(u) = 1$ if $u \in A$ and $0$ otherwise.
Given two distributions $P$ and $Q$, defined over sets $\mathcal{Z}_P$ and $\mathcal{Z}_Q$, respectively, $\Pi(P,Q)$ denotes the set of joint distributions over $\mathcal{Z}_P \times \mathcal{Z}_Q$ such that its marginal over $\mathcal{Z}_P$ and $\mathcal{Z}_Q$ is equal to $P$ and $Q$, respectively. The 2-Wasserstein cost between distributions $P_Y$ and $Q_Y$ on $Y$ is defined as:
\begin{equation}
    W_2(P_Y,Q_Y)^2 := \inf_{(Y,Y') \sim M\in\Pi(P_Y,Q_Y)}\: \mathbb{E}_M\bigl[ \Vert Y-Y'\Vert^2_2\bigr],
\end{equation}
where $Y,Y'$ are constrained to be marginally distributed as $P,\, Q$, respectively. To extend this definition to the supervised learning setting, for joint distributions $P_{X,Y}$ and $Q_{X,Y}$ sharing the same marginal $P_X$ we define:
\begin{equation}
    W_2(P_{X,Y},Q_{X,Y}) := \mathbb{E}_{P_X}\bigl[W_2(P_{Y|X=x},Q_{Y|X=x}) \bigr].
\end{equation}
\section{Problem Formulation}\label{sec:formulation}
We consider the mixed linear regression problem, where the output to each input vector is generated by one of $k$ linear regression models. Specifically, we observe data points $\mathcal{S}:= \{(x_i, y_i)\}_{i=1}^n$ where, for every $i$, $x_i \in \mathcal{X} \subset \reals^d$, $y_i \in \reals$, and 
\begin{equation}\label{eq:model}
    y_i = \sum_{j=1}^k \indicator{z_i = j} (\beta_j^*)^\top x_i + \epsilon_i, \quad i = 1, ..., n,
\end{equation}
with latent variable $z_i \in \{1, 2, ..., k\}$. Each $\beta_j^* \in \reals^d$ denotes the regression vector for one of the overall $k$ components.
We assume that the input data $\{x_i\}_{i=1}^n$ are norm-bounded random vectors with $x_i$ drawn i.i.d. from $p_x$ with $\sup_{x \in \mathcal{X}}~ \|x \| \leq C$, that the noises $\{\epsilon_i\}_{i=1}^n$ are independent of the input data and drawn i.i.d. from the normal distribution $p_\epsilon := \normal(0, \sigma^2)$
where $\sigma^2$ is known,
and that each $z_i$ is drawn from $\{1, 2, ..., k\}$ uniformly at random.

The MLR problem is to find the distribution $p^\star$ that best fits the data $\mathcal{S}$ (according to some metric). We know that $p^\star$ lies in the class of distributions $\mathcal{P}$, parameterized by $\beta_{[k]}:=(\beta_j)_{j=1}^k$:
%%%%%%%%%
\begin{align}\label{Class_P}
\mathcal{P} & := \bigg\{p_{\beta_{[k]}}(X,Y): X \sim p_x, \, Z \sim \text{Unif}(\{1,...,k\}), \nonumber \\
& \qquad\Pr{Y \mid X = x, Z=j} \eqdist \normal(\beta_j^\top x, \sigma^2) \bigg\}.  
\end{align}
%%%%%%%%%%

The Expectation Maximization algorithm (EM) is commonly used to tackle this problem. The EM algorithm provides a widely-used heuristic for computing the maximum likelihood estimator (MLE) for the regressors $(\beta_j^*)_{j=1}^k$. However, implementing the EM algorithm in the federated learning setting can be challenging. We consider finding the $\beta_{[k]}$ which minimizes the distance between $p_{\beta_{[k]}}$ and $p_{\beta_{[k]}}^*$ with respect to a distribution distance measure, i.e.,
\begin{equation*}
 \argmin_{\beta_{[k]}} D\left({p_{\beta_{[k]}^*}},{p_{\beta_{[k]}}}\right),   
\end{equation*}
where $D(\cdot,\cdot)$ is a distribution distance metric to be chosen. In this work, we use the expected 2-Wasserstein cost as our metric, resulting in the problem 
\begin{equation}\label{eq:wass_minimization}
 \argmin_{\beta_{[k]}} W_c(p_{\beta_{[k]}^*},p_{\beta_{[k]}})   
\end{equation}
It is worth noting that, here, and similar to well-known EM analysis \cite{kwon2019global}, we assume $\sigma$ is known to simplify the derivations. However, we could extend our framework to the case that $\sigma$ is not known by parametrizing $\mathcal{P}$ by both $\beta_{[k]}$ and $\sigma$ and minimzing over both of them in \eqref{eq:wass_minimization}. Furthermore, as we will see in Section \ref{sec:theory}, one advantage of our proposed method works without the knowledge of $\sigma$ for the symmetric case with $k=2$.

In the next section, we use the properties of 2-Wasserstein distance to build a minimax framework for mixed linear regression and then show how it can be used in the federated learning setting.

%%%%%%% Remark %%%%%%
% \begin{remark} \label{remark:EM}
\iffalse
To see the relation between \eqref{eq:wass_minimization} and the maximum likelihood approach, consider the MLE estimator $\betaML$ which is the global maximum of the likelihood function:
\begin{align*}
\betaML & := \argmax_{\beta_{[k]}} \prod_{i=1}^n p_{\beta_{[k]}}(x_i,y_i) \\
%& = \argmin_{\beta_{[k]}} \frac{1}{n} \sum_{i=1}^n -\log{p_{\beta_{[k]}}(x_i,y_i)} \\
& = \argmin_{\beta_{[k]}} \E_{\hat{p}}[\log \hat{p}(x,y) -\log p_{\beta_{[k]}}(x,y)] \nonumber \\
&= \argmin_{\beta_{[k]}} \dkl{\hat{p}}{p_{\beta_{[k]}}},
\end{align*}
where $\hat{p}$ denotes the empirical distribution, i.e the uniform distribution over $\mathcal{S}$, and $\dkl{\cdot}{\cdot}$ is the well-known KL divergence. The above equations rederives the well-known fact that $\betaML$ has the minimum distance, in terms of KL divergence, to the empirical distribution of samples. From this viewpoint of distance minimization (\ie projection of the observed distribution onto the parameterized set of distributions $\mathcal{P}$), the EM approach provides an optimization algorithm to minimize the KL Divergence, whereas our framework provides an algorithm to minimize the Wasserstein distance.
\fi
\section{A Wasserstein Minimax Approach to MLR}\label{sec:algorithm}

To formulate a minimax learning problem, we 
replace the Wasserstein cost in \eqref{eq:wass_minimization} with its dual representation according to the Kantorovich duality \cite{villani2008optimal}. This reformulation results in the following minimax optimization problem:
\begin{equation}\label{eq:kantorovich_duality}
    \argmin_{\beta_{[k]}} \max_{\psi}\: \mathbb{E}_{p_{\beta_{[k]}^*}}[\psi(x,y)] -  \mathbb{E}_{p_{\beta_{[k]}}}[\psi^c(x,y)],
\end{equation}
where the optimization variable $\psi: \reals^d \times \reals \to \reals$ is an unconstrained function, and the c-transform $\psi^c(x,y)$ is defined as
\begin{equation} \label{eq:c-transform}
    \psi^c(x,y) = \sup_{y'} \psi(x,y') -
    \frac{1}{2} \Vert {y}-{y}'\Vert^2_2.
\end{equation}
Note that the two distributions $p_{\beta_{[k]}}$ and $p_{\beta_{[k]}^*}$ have the same marginal $p_x$ and only differ in the conditional distribution $p_{y\mid x}$. As a result, the optimal transport task requires to only move mass to match the conditional distribution. This observation results in the cost function used to define the c-transform operation in \eqref{eq:c-transform}.

However, the above optimization problem for an unconstrained $\psi$ is known to be statistically and computationally complex \cite{arora2017generalization}. In this section, our goal is to show that one can solve \eqref{eq:kantorovich_duality} over the following space of functions for $\psi$ parameterized by $2k$ vectors ${\gamma_{[2k]}}\in  \mathcal{F}$, with 
\begin{align*}
 \mathcal{F}= \biggl\{ \psi_{{\gamma_{[2k]}}}:\, & \psi_{{\gamma_{[2k]}}}(x,y)= \\ & \log\biggl(\frac{\sum_{i=1}^k\exp\bigl(\frac{-1}{2\sigma^2}(y-\gamma^\top_{2i-1}x )^2\bigr)}{\sum_{i=1}^{k}\exp\bigl(\frac{-1}{2\sigma^2}(y-\gamma^\top_{2i} x)^2\bigr)}\biggr)\biggr\}.
\end{align*}
This provides a tractable minimax optimization problem whose solution is provably close to that of \eqref{eq:kantorovich_duality}.
To find the above parameterized space for $\psi$, we apply Brenier's theorem connecting the optimal $\psi$ to a transport map between two MLR models.

\begin{lemma}[Brenier's Theorem, \cite{villani2008optimal}]\label{lemma:breniers_thm}
Assume $X \sim p_x$, and consider random variables $Y$ and $Y'$ such that $(X,Y) \sim {p_{\beta_{[k]}^*}}$ and $(X,Y') \sim {p_{\beta_{[k]}}}$ provide two MLR models according to $\beta_{[k]}^*$ and $\beta_{[k]}$, respectively. Then, 
the optimal $\psi$ in \eqref{eq:kantorovich_duality} satisfies the following transportation property 
\begin{equation}\label{eq:breniers}
    (X,Y - \psi_y(X,Y)) \eqdist (X,Y'),
\end{equation}
where $\psi_y(x,y) := \frac{\partial}{ \partial y} \psi(x,y)$.
\end{lemma}

The above lemma shows that the optimal transport map's derivative will transport samples between the two domains. Therefore, we need to characterize the potential optimal transport maps and consider their integral for constraining $\psi$. To do this, we find an approximation of this optimal mapping in two steps: 
First, we use a randomized technique, adapted from \cite{farnia2020gat}, to come up with a mapping $\Psi$ that maps $(X,Y,Z)$ to $(X,Y')$ where $Z$ is the regression index for $(X,Y)$. 
Then, we obtain $\tilde{\Psi}$ by taking the expectation of $\Psi$ with respect to $Z$ to drop the dependence of $Z$. 
We bound the error of this approximation step in Theorem \ref{theorem:bound_error}.

For the first step, consider the following randomized transportation map:
\begin{equation}
    \Psi(X,Y,Z) :=  Y+\sum_{i=1}^k \indicator{Z=i}(\beta_i-\beta_i^*)^\top X ,
\end{equation}
where $Z$ denotes the regression model index in the first mixture $(X,Y)$. Note that the above randomized map will transport samples between the two MLR distributions, i.e., 
$$\bigl( X, Y+\sum_{i=1}^k \indicator{Z=i}(\beta_i-\beta_i^*)^\top X\bigr)  \eqdist \bigl(X,Y').$$ 
However, the above mapping is a randomized function of $x,y$ since $Z$ remains random after observing the outcome for $x,y$. To obtain a deterministic map $\tilde{\Psi}: \reals^d \times \reals \to  \reals$ from this randomized map, we consider its conditional expectation given $(X,Y)$:
\begin{align}\label{Eq: approximate transport map}
\tilde{\Psi} & (x,y) := \mathbb{E}\bigl[\Psi(X,Y,Z)|X=x,Y=y \bigr]\nonumber\\
&= y+ \sum_{i=1}^k \Pr {Z=i|X=x,Y=y}(\beta_i-\beta_i^*)^\top x.
\end{align}
In the above equation, by Bayes' rule we have
\begin{align*}
  & \Pr{Z=i|X=x, Y=y} \\
  & =\frac{\exp(\frac{-1}{2\sigma^2}(y-(\beta_i^*)^\top x)^2)}{\sum_{j=1}^k \exp(\frac{-1}{2\sigma^2}(y-(\beta_j^*)^\top x)^2)} \\
  & = \frac{\exp(\frac{-1}{2\sigma^2}y(\beta_i^*)^\top x)\exp(\frac{-1}{2\sigma^2}((\beta_i^*)^\top x)^2)}{\sum_{j=1}^k \exp(\frac{-1}{2\sigma^2}y(\beta_j^*)^\top x)\exp(\frac{-1}{2\sigma^2}((\beta_j^*)^\top x)^2)}. 
\end{align*}
Note that if $\tilde{\Psi}$ was the optimal transport with $\frac{\partial}{\partial y}\tilde{\psi}(x,y)= \tilde{\Psi}(x,y)$, then
\begin{equation*}
W_2\bigl(p_{\beta_{[k]}},p_{\beta_{[k]}^*}\bigr) = \mathbb{E}_{p_{\beta_{[k]}^*}}[\tilde{\psi}(x,y)] -  \mathbb{E}_{p_{\beta_{[k]}}}[\tilde{\psi}^c(x,y)]     
\end{equation*}
With $\tilde{\Psi}$ as an approximate solution, we next state the following result which bounds the duality gap of $\tilde{\Psi}$.
\begin{theorem}
\label{theorem:bound_error}
Let $\tilde{\psi}:\mathbb{R}^d\times \mathbb{R}\rightarrow \mathbb{R}$ be a convex function such that for every $x\in\mathcal{X}$, $\frac{\partial}{\partial y}\tilde{\psi}(x,Y)$ shares the same distribution with $\tilde{\Psi}(x,Y)$ in \eqref{Eq: approximate transport map}\footnote{The existence of $\tilde{\psi}$ is guaranteed based on Brenier's theorem.}.
Assume that for every $x \in \mathcal{X}$ and every $\beta_i$ we have $|(\beta_i^*)^{\top} x|\le C'$. For an observation of input and output $(X,Y)$ with regression index $Z$, we denote the optimal Bayes classifier of the cluster of $(X,Y)$ as $Z^*(X,Y)$.
Let $P_{\operatorname{err}}:= \Pr{Z\neq Z^*(X,Y)}$ be the probability error of the Bayes classifier.
Then, we have:
\begin{align*}
    0 &\le \! W_2\bigl(p_{\beta_{[k]}},p_{\beta_{[k]}^*}\bigr) \! - \mathbb{E}_{p_{\beta_{[k]}^*}}[\tilde{\psi}(x,y)] +  \mathbb{E}_{p_{\beta_{[k]}}}[\tilde{\psi}^c(x,y)] \\
    &\le\, 16(C'^2+2\sigma^2) \sqrt{P_{\operatorname{err}}} + 2(C'^2+\sigma^2) \sqrt[4]{P_{\operatorname{err}}}. 
\end{align*}
\end{theorem}
\begin{proof}
See Appendix \ref{proof-theorem:bound_error}.
\end{proof}
Finally, we estimate $\tilde{\psi}$ with a function from $\mathcal{F}$ that does not depend on the optimal $\beta^\star_{[k]}$.
%%%%%%%%%%%%%%%%%%%%%%%%%%%%%%%%%%%%%%%%%%
%%%%%%%%%%%%%%%%%%%%%%%%%%%%%%%%%%%%%%%%%%
\begin{proposition}\label{prop:psi}
\label{proposition:approximation}
Assume that $\sum_{i=1}^k \vert \mathbb{P}_{(\beta_j)_{j=1}^k}(Z=i|X=x,Y=y)- \mathbb{P}_{(\beta^*_j)_{j=1}^k}(Z=i|X=x,Y=y)\vert \le \delta$
and $\max_i \vert \beta_i^{\top}x\vert, \max_i \vert (\beta^*)_i^{\top}x\vert \le C'$ for every $x\in \mathcal{X}$ and feasible $\beta_i$. 
Then, there exists $(\gamma_i)_{i=1}^{2k}$ such that the function
\begin{equation} \label{eq:prop_psi_k}
    \psi_{{\gamma_{[2k]}}}(x,y)=\log\biggl(\frac{\sum_{i=1}^k\exp\bigl(\frac{-1}{2\sigma^2}(y-\gamma^\top_{2i-1}x)^2\bigr)}{\sum_{i=1}^{k}\exp\bigl(\frac{-1}{2\sigma^2}(y-\gamma^\top_{2i}x)^2\bigr)}\biggr),
\end{equation}
approximates $\tilde{\psi}$, with error bounded by $C'\delta$.
\end{proposition}
%%%%%%%%%%%%%%%%%%%%%%%%%%%%%%%%%%%%%%%%%%
\begin{proof}
See Appendix \ref{proof-proposition:approximation}.
\end{proof}
%%%%%%%%%%%%%%%%%%%%%%%%%%%%%%%%%%%%%%%%%%
Combining \eqref{eq:kantorovich_duality} and Proposition \ref{prop:psi}, we formulate the following minimax problem which approximates \eqref{eq:kantorovich_duality}:
\begin{align}\label{eq:W2_opt_problem_no_reg1}
\min_{\beta_{[k]}}\: \max_{\gamma_{[2k]}}\: \mathbb{E}_{p_{\beta_{[k]}^*}}[\psi_{{\gamma_{[2k]}}}(x,y)] - \mathbb{E}_{p_{\beta_{[k]}}}[\psi^c_{{\gamma_{[2k]}}}(x,y)].
\end{align}
By Proposition \ref{prop:psi} and Theorem \ref{theorem:bound_error}, the approximation error is bounded when the clusters can be identified with high precision by the optimal Bayes classifier. This condition can be thought of as a separability condition.

%%%%%%%%%%%%%%%%%%%%%%%%

\subsection{Reducing c-transform to Norm Regularization}
In order to simplify the c-transform operation, we introduce a regularization penalty term to substitute the c-transform term in \eqref{eq:W2_opt_problem_k}. To do this, we bound the expected value of the $c$-transform $\psi^c(x, y)$ \eqref{eq:c-transform} by the expectation of $\psi(x, y)$ and a regularization term. This bound, given in the following proposition, allows us to formulate a strongly-concave maximization problem.

\begin{proposition}\label{prop:bounding_c_transform}
Consider the discriminator function $\psi_{\gamma_{[2k]}}(x,y)$ in \eqref{eq:psi_2} and recall that $\|x\| \leq C$. Assume that $2kC^2\max_i\|\gamma_i\|^2 \leq \eta < 1$. Then, for any set of vectors $\gammatilde_{[2k]} \in \reals^{2k\times d}$, we have
\begin{align}\label{eq:prop1:regularization}
    \Ex{\psi^c_{\gamma_{[2k]}}({x},y)} &
    \leq \Ex{\psi_{\gamma_{[2k]}}({x},y)}+ \frac{kC^2\Ex{(1 + C\abs{y})^2}}{1-\eta} \nonumber \\
    &\quad\times\left( \sum_{i=1}^{k}\|\gamma_i - \gammatilde_i\|^2 +\|\gamma_{i+k} - \gammatilde_i\|^2\right).
\end{align}
\end{proposition}
%%%%%%%%%%%%%%%%%%%%%%%%%%%%%%%%
\begin{proof}
See Appendix \ref{proof-prop:bounding_c_transform}.
\end{proof}

\subsection{WMLR Algorithm}
It can be seen that \eqref{eq:W2_opt_problem_no_reg1} represents a nonconvex-nonconcave optimization problem.
As shown in Proposition \ref{prop:bounding_c_transform}, we could bound the $c$-transform by adding a regularization, and, as a result, we obtain the following nonconvex strongly-concave minimax problem
\begin{align}
    \label{eq:W2_opt_problem_k}
    \min_{\beta_{[k]}}\: &\max_{\gamma_{[2k]}}\:  \matL(\beta_{[k]}, \gamma_{[2k]}) := \nonumber\\
    &\qquad \mathbb{E}_{p_{\beta_{[k]}^*}}[\psi_{{\gamma_{[2k]}}}(x,y)] - \mathbb{E}_{p_{\beta_{[k]}}}[\psi_{{\gamma_{[2k]}}}(x,y)]\nonumber\\
    &\qquad- \lambda\left(\sum_{i=1}^{k}\|\gamma_i - \gammatilde_i\|^2 +\|\gamma_{i+k} - \gammatilde_i\|^2\right),
\end{align}
where $\gammatilde$ is a properly chosen reference vector. 

Since we do not have access to $p_{\beta_{[k]}^*}$ in practice, we replace $\mathbb{E}_{p_{\beta_{[k]}^*}}[\psi_{{\gamma_{[2k]}}}(x,y)]$ above with 
$\mathbb{E}_{\hat{p}}[\psi_{{\gamma_{[2k]}}}(x,y)]$ where $\hat{p}$ is the empirical problem over the observed dataset $\mathcal{S} = \{(x_1,y_1),\ldots, (x_n, y_n)\}$. We denote the resulting function by $\hatL(\beta_{[k]}, \gamma_{[2k]})$.

WMLR, given in Algorithm \ref{alg:WMLR}, uses GDA to solve \eqref{eq:W2_opt_problem_k}. Later, we show that solving \eqref{eq:W2_opt_problem_k} can recover the underlying $\beta_{[k]}$ that solves the original unregularized \eqref{eq:W2_opt_problem_no_reg1}.

\begin{algorithm}[tb]
   \caption{WMLR}
   \label{alg:WMLR}
\begin{algorithmic}
   \STATE {\bfseries Input:} $(x_{i}, y_{i})_{i \in [n]}$, $\beta_{[k]}^{(0)}, \gamma_{[2k]}^{(0)}$, step sizes $\alpha_\mathrm{min}, \alpha_\mathrm{max}$
   \FOR{ $t  = 0$ {\bfseries to} $T-1$ }
    \FOR {$i = 1${ \bfseries to} $k$ }
   \STATE $\beta^{(t+1)}_i = \beta^{(t)}_i - \alpha_\mathrm{min} \nabla_{\beta_i} \hatL(\beta_{[k]}^{(t)}, \gamma_{[2k]}^{(t)})$
   \STATE $\gamma^{(t+1)}_{i} = \gamma^{(t)}_{i} + \alpha_\mathrm{max} \nabla_{\gamma_{i}} \hatL(\beta_{[k]}^{(t)}, \gamma_{[2k]}^{(t)})$
   \STATE $\gamma^{(t+1)}_{i+k} = \gamma^{(t)}_{i+k} + \alpha_\mathrm{max} \nabla_{\gamma_{i+k}} \hatL(\beta_{[k]}^{(t)}, \gamma_{[2k]}^{(t)})$
    \ENDFOR
   \ENDFOR
\end{algorithmic}
\end{algorithm}

\paragraph{Federated Learning} Since we use the gradient-based GDA algorithm to solve the minimax optimization problem, WMLR is particularly amenable to distributed computation. Here, we consider a federated learning setting with $M$ agents, where each agent $m$ has data samples $(x_{i,m}, y_{i,m})_{m \in [M],\, i \in [N]}$. This setting can model both the following scenarios: 1) each sample belongs to any of the $k$ components with equal probability, as in the centralized case; or 2) all the samples for each individual agent are associated with the same cluster.
The latter scenario arises when an unknown latent variable governs the regression model that best describes the relationship between $y$ and $x$.
We note that our proposed algorithm can apply to both these cases.
Every agent only has access to its own data and therefore can only estimate its own minimax objective $\hatL_m$. Therefore, the total minimax objective in the network will be
\begin{equation}
    \hatL(\beta_{[k]}^{(t)}, \gamma_{[2k]}^{(t)}) = \frac{1}{M}\sum_{m=1}^M \hatL_m(\beta_{[k]}^{(t)}, \gamma_{[2k]}^{(t)}),
\end{equation}
where $\hatL_m$ computes $\mathbb{E}_{\hat{p}}$ using only the data on agent $m$. Our Federated WMLR (F-WMLR) algorithm adds a communication step after each GDA iteration, as described in Algorithm \ref{alg:F-WMLR}. This algorithm could be extended to include multiple GDA steps or partial agent participation at each round.  We show that F-WMLR enjoys the same theoretical guarantees as WMLR in Section \ref{sec:theory} below.

\begin{algorithm}[tb]
   \caption{F-WMLR}
   \label{alg:F-WMLR}
\begin{algorithmic}
   \STATE {\bfseries Input:} $(x_{i,m}, y_{i,m})_{m \in [M],\, i \in [N]}$, $\beta_{[k]}^{(0)}, \gamma_{[2k]}^{(0)}$, step sizes $\alpha_\mathrm{min}, \alpha_\mathrm{max}$.
   \FOR{ $t  = 0$ { \bfseries to} $T-1$ }
   \STATE Broadcast $\beta_{[k]}^{(0)}, \gamma_{[2k]}^{(0)}$ to all agents
   \FOR{each agent $m=1$ {\bfseries to} $M$}
    \FOR {$i = 1${ \bfseries to} $k$ }
   \STATE $\beta^{(t+1)}_{i,m} = \beta^{(t)}_i - \alpha_\mathrm{min} \nabla_{\beta_i} \hatL_m(\beta_{[k]}^{(t)}, \gamma_{[2k]}^{(t)})$
   \STATE $\gamma^{(t+1)}_{i,m} = \gamma^{(t)}_{i} + \alpha_\mathrm{max} \nabla_{\gamma_{i}} \hatL_m(\beta_{[k]}^{(t)}, \gamma_{[2k]}^{(t)})$
   \STATE $\gamma^{(t+1)}_{i+k,m} = \gamma^{(t)}_{i+k} + \alpha_\mathrm{max} \nabla_{\gamma_{i+k}} \hatL_m(\beta_{[k]}^{(t)}, \gamma_{[2k]}^{(t)})$
    \ENDFOR
   \STATE Send $\beta_{[k], m}^{(t+1)}, \gamma_{[2k], m}^{(t+1)}$ to server
   \ENDFOR
   \STATE Collect $\beta_{[k], m}^{(t+1)}, \gamma_{[2k], m}^{(t+1)}$ from all agents $m \in [M]$
    \FOR {$i = 1${ \bfseries to} $k$ }
   \STATE $\beta^{(t)}_i  = \frac{1}{M}\sum_{m=1}^M \beta^{(t+1)}_{i,m}$
   \STATE $\gamma^{(t)}_i  = \frac{1}{M}\sum_{m=1}^M \gamma^{(t+1)}_{i,m}$
   \STATE $\gamma^{(t)}_{i+k}  = \frac{1}{M}\sum_{m=1}^M \gamma^{(t+1)}_{i+k,m}$
    \ENDFOR
   \ENDFOR
\end{algorithmic}
\end{algorithm}

\subsection{Generalization to Non-linear Models}
The WMLR algorithm can also be used for the setting where the output is a mixture of linear regressions of a nonlinear transformation of the input vector that is common to all components. The corresponding $\psi$ function will be
\begin{equation}\label{eq:psi_nonlinear}
\psi_{\phi,{\gamma_{[2k]}}}(x,y)= \log\biggl(\frac{\sum_{i=1}^k\exp\bigl(\frac{-1}{2\sigma^2}(y-\gamma^\top_{2i-1} \phi(x))^2\bigr)}{\sum_{i=1}^{k}\exp\bigl(\frac{-1}{2\sigma^2}(y-\gamma^\top_{2i} \phi(x))^2\bigr)}\biggr).
\end{equation}

Our theoretical results, discussed in the subsequent section, do not extend to the nonlinear case in general. However, WMLR will still convergence to a minimax stationary point when $\psi$ has the form \eqref{eq:psi_nonlinear}.

\section{Convergence Guarantees for WMLR}\label{sec:theory}
In this section, we focus on the case $k=2$, and further explore the minimax formulation \eqref{eq:W2_opt_problem_k}. In particular, to simplify the derivations, we focus on the \textit{symmetric} case, i.e., $\beta_2^* = -\beta_1^*$, which has been studied in the in EM literature as well (see \cite{kwon2019global} and references therein).The non-symmetric case can be reduced to the symmetric case, by first estimating $\bar{\beta}$ as the mean of $\beta^*_{[2]}$, and then replacing each data point $(x_i, y_i)$ by $(x_i, y_i - \bar{\beta}^\top x_i)$. 

In the symmetric setting, we have that $\gamma_3 = -\gamma_1$ and $\gamma_4=-\gamma_2$ in $\psi_{{\gamma_{[4]}}}$ and $\beta_2 = -\beta_1$ in $p_{\beta_{[2]}}$ . 
We next observe that, in this case, $\psi_{\gamma_{[4]}}$ can be decomposed into the following two terms
\begin{align*}
    &\psi_{\gamma_{[4]}}(x,y)\\
    &\,= \log\biggl(\frac{\exp\bigl(\frac{-1}{2\sigma^2}(y-\gamma^\top _1x)^2\bigr)+\exp\bigl(\frac{-1}{2\sigma^2}(y+\gamma^\top _1x)^2\bigr)}{\exp\bigl(\frac{-1}{2\sigma^2}(y-\gamma^\top _2x)^2\bigr)+\exp\bigl(\frac{-1}{2\sigma^2}(y+\gamma^\top _2x)^2\bigr)}\biggr)
\end{align*}
\begin{align*}
    &\,= x^\top Ax + \log\left({\exp(\frac{y\gamma^\top _1x}{2\sigma^2})+\exp(\frac{-y\gamma^\top _1x}{2\sigma^2})}\right) \\
    &\qquad- \log \left({\exp(\frac{y\gamma^\top _2x}{2\sigma^2})+\exp(\frac{-y\gamma^\top _2x}{2\sigma^2})} \right),
\end{align*}
where $A:=\tfrac{\gamma_1\gamma^\top _1-\gamma_2\gamma^\top _2}{2\sigma^2}$. 

Since the marginal distribution of $p_{\beta_{[k]}}$ over $X$ is constant (and equal to $p_x$), we can ignore the quadratic term $x^\top Ax $ as it will be canceled out in 
$\mathbb{E}_{p_{\beta_{[2]}^*}}[\psi_{{\gamma_{[4]}}}(x,y)] - \mathbb{E}_{p_{\beta_{[2]}}}[\psi^c_{{\gamma_{[4]}}}(x,y)]$. 
Furthermore, we can absorb $2\sigma^2$ into $\gamma_1$ and $\gamma_2$.
Thus, we can replace $\psi_{\gamma_{[4]}}$ in \eqref{eq:W2_opt_problem_k} with $k=2$ by
\begin{align}\label{eq:psi_2}
&\psi_{\gamma_1, \gamma_2}(x,y):= \log\left({\exp(y\gamma^\top _1x)+\exp(-y\gamma^\top _1x)}\right) \nonumber \\&\qquad - \log \left({\exp(y\gamma^\top _2x)+\exp(-y\gamma^\top _2x)} \right).    
\end{align}
As a result, and in this section, we work with $\psi_{\gamma_1, \gamma_2}(x,y)$ instead of $\psi_{\gamma_{[4]}}$ in \eqref{eq:W2_opt_problem_k}. Also, we simplify $\matL(\beta_{[2]}, \gamma_{[4]})$ and $\hatL(\beta_{[2]}, \gamma_{[4]})$ by $\mathcal{L}(\beta, \gamma_1, \gamma_2)$ and $\hatL(\beta, \gamma_1, \gamma_2)$, respectively.

Our goal is to solve the minimax problem \eqref{eq:W2_opt_problem_k} to a \textit{minimax stationary point}, which we define below.

\begin{definition}\label{def:eps_stationary_point}
Consider a function $f(x,y)$, where $f(x, \cdot)$ is strongly concave for all $x$. The point
$x^\star$ is an $\epsilon$ minimax stationary point of
\begin{equation}
    \min_x \max_y f(x,y)
\end{equation}
if 
$\| \nabla_x F(x^\star) \| \leq \epsilon$, and $F(x) = \max_y f(x,y)$.
\end{definition}
To discuss the convergence to stationary points in our setting, we define 
\begin{equation} \label{eqn:def_L}\begin{split}
\mathcal{L}(\beta) &:= \max_{\gamma_{[2]}} \mathcal{L}(\beta, \gamma_1, \gamma_2) \\
\hatL(\beta) &:= \max_{\gamma_{[2]}} \hatL(\beta, \gamma_1, \gamma_2). 
\end{split}
\end{equation}
The outline of our theoretical results is as follows: We first show that the added regularization term forms a strongly concave inner maximization problem, and using that, in Theorem \ref{theorem:convergence_to_stationary_pt}, we show WMLR finds the minimax stationary point solution of $\hatL$. Next, in Theorem  \ref{theorem:equal_to_EM}, we show that under certain assumptions, this solution is optimal $\beta^*$. Finally, we provide bounds on the generalization error as well.

\subsection{Local and Global Convergence of WMLR}
In this subsection, we show that the GDA algorithm is guaranteed to converge to the optimal solution to \eqref{eq:W2_opt_problem_k}. 

\begin{theorem}\label{theorem:convergence_to_stationary_pt}
    Consider the minimax problem \eqref{eq:W2_opt_problem_k}. Assume that $C^2 \E_{\hat{p}} [ y^2] \leq \eta < \frac{\lambda}{2}$ and $C^2 < \frac{\lambda}{2}$. Then the WMLR algorithm (Algorithm \ref{alg:WMLR}) with step sizes $\alpha_\mathrm{max} = \frac{1}{L}$ and $\alpha_\mathrm{min} = \frac{1}{\kappa^2L}$ for $L = \lambda + 4\eta(1 + \eta/\lambda + \|\gammatilde\|)$ and $\kappa = \frac{L}{\lambda - 2\eta}$ will find an $\epsilon$-approximate stationary point in the following number of iterations:  
    $$\mathcal{O}\left(\frac{\kappa^2L \Delta + \kappa L^2(2\eta/\lambda)^2}{\epsilon^2}\right),$$
    where $\Delta := \hatL(\beta^{(0)}) - \min_\beta \hatL(\beta)$.
\end{theorem}
%%%%%%%%%%%%%%%%%%%%%%%%%
\begin{proof}
See Appendix \ref{proof-theorem:convergence_to_stationary_pt}.
\end{proof}

\begin{remark} \label{corr:NN} Consider the minimax problem \eqref{eq:W2_opt_problem_k} where $x$ is replaced by non-linear $\phi(\cdot;w)$, a neural network parameterized by weights $w$. The weights $w$ appear in the minimization problem; hence, the problem remains nonconvex strongly-concave and the guarantee in Theorem \ref{theorem:convergence_to_stationary_pt} also applies to the non-linear case, \ie WMLR still results in an approximate stationary point.
\end{remark}

In Theorem \ref{theorem:equal_to_EM} below, we show global convergence under correlated projections along $\beta^*$ and $\tilde{\gamma}$.
\begin{theorem}\label{theorem:equal_to_EM}
    Consider two symmetric components for feature variables $x$. Suppose that the variables $\tilde{\gamma}^{\top}x$ and ${\beta^*}^{\top}x$ are correlated enough such that
    $$\max\bigl\{ \Pr{\tilde{\gamma}^{\top}xx^{\top}\beta^* \le 0}\, ,\,  \Pr{\tilde{\gamma}^{\top}xx^{\top}\beta^* \ge 0} \bigr\} = 1.$$
    Then, any stationary minimax solution $\widehat{\beta}$ to the minimiax problem \eqref{eq:W2_opt_problem_k} which satisfies the above condition will further provide a global minimax solution to \eqref{eq:W2_opt_problem_k}.
\end{theorem}
%%%%%%%%%%%%%%%%%%%%%%%%%%%
\begin{proof}
See Appendix \ref{proof-theorem:equal_to_EM}.
\end{proof}

The above theorem shows that if $\tilde{\gamma}$ and $\beta^*$ are sufficiently aligned such that the random variables $\tilde{\gamma}^{\top} x$ and ${\beta^*}^{\top} x$ are correlated enough, then a stationary minimax point for the WMLR's minimax problem further leads to a global solution to the WMLR problem.

Note that the condition in the theorem statement automatically holds for a $1$-dimensional scalar $x$. In general, the theorem condition suggests that we need to chose the reference vector $\tilde{\gamma}$ almost aligned to $\beta^*$. One way to do so is as follows: First, note that $\beta^*_{\text{norm}} := \beta^*/\|\beta^*\|$ is the top eigenvector of $\mathcal{M}:= \E_x[(x^\top \beta^*)^2xx^\top]$. Let us assume the top eigenvector of $\mathcal{M}$ is unique, i.e., $\beta^*_{\text{norm}}$is the only eigenvector corresponding to the maximum eigenvalue of $\mathcal{M}$.  In that case, it can be shown that for sufficiently large $n$, $\beta^*_{\text{norm}}$ is approximately the top eigenvector of $M_n:=\frac{1}{n} \sum_{i=1}^n y_i^2 x_i x_i^\top$. To see this, we need to show the solution to $\argmax_{v: \|v\|=1} v^\top M_n v$ is close to $\beta^*_{\text{norm}}$. To do so, first note that by classic concentration bounds we could show that, for sufficiently large $n$, $v^\top M_n v$ is close to $\E[\|v^\top (xy)\|^2]$. That said, maximizing $\E[\|v^\top (xy)\|^2]$ over $v$ is equivalent to maximizing $\E[\|v^\top (xx^\top \beta^*)\|^2] = v^\top \E_x[(x^\top \beta^*)^2xx^\top] v = v^\top \mathcal{M} v$ over $v$, and we assumed $\beta^*_{\text{norm}}$ is the unique solution to the latter problem.  
We further evaluate this choice of refrence vector in the our numerical experiments.

\begin{remark} \label{corr:FL} (Federated Learning)
F-WMLR (Algorithm \ref{alg:F-WMLR}) will produce the same sequence of iterates as the centralized WMLR algorithm by linearity of the gradient operator. Therefore, the above convergence results for WMLR will also apply to F-WMLR.
\end{remark}

\subsection{Generalization of WMLR}
Here we establish generalization error bounds for the convergence of the value and gradient of the empirical objective to those of the underlying distribution.
\begin{theorem}
Recall the definition of $\matL(\beta)$ and $\hatL(\beta)$ \eqref{eqn:def_L}. Consider the minimax mixed regression setting with norm-bounded random vector $X, \; \Vert X \Vert_2 \le C$ and noise vector $\epsilon\sim\mathcal{N}(0,\sigma^2)$. Assume that $\max\{C,\sigma\} \le 1$. Then, we have the following generalization bounds hold with probability at least $1-\delta$ for every $\Vert \beta \Vert_2\le \eta$:
   \begin{align*}
    |\mathcal{L}(\beta) - \widehat{\mathcal{L}}(\beta) | &\le O\bigl(\sqrt{\frac{d\eta^4\log(\eta/\lambda\delta)}{n}}\bigr) , \\ 
    \Vert\nabla\mathcal{L}(\beta)-\nabla\widehat{\mathcal{L}}(\beta)\Vert_2 &\le O\big(\sqrt{\frac{d\eta^4 \log(\eta/{\lambda}\delta)}{(1-\eta/\lambda)^2n}} \bigr).
\end{align*}

\end{theorem}
\section{Numerical Experiments}
We consider $k=2$ and focus on the symmetric case with $\beta_2^* = -\beta_1^*$
for the numerical experiments. We implement\footnote{\url{https://github.com/tjdiamandis/WMLR}.} Algorithms \ref{alg:WMLR} and \ref{alg:F-WMLR} in Section \ref{sec:algorithm} for both the centralized and federated learning settings. In both settings, we run experiments for a high SNR (10) and a low SNR (1) regime. We include two additional SNRs in the federated experiments.
We set $d=128$, draw $x_i$ from $\normal(0, I)$, set noise variance $\sigma^2 = 1$, and draw $\beta^\star$ uniformly at random from the spherical shell $\mathcal{S}_\mathrm{SNR} = \{z \mid \|z\| = \mathrm{SNR}\}$. We search over regularization parameter $\lambda$, and step sizes are $\alpha_\mathrm{max} = {1}/{2\lambda}$ and $\alpha_\mathrm{min} = {\alpha_\mathrm{max}}/{10}$ (motivated by Theorem \ref{theorem:convergence_to_stationary_pt}). Note that the algorithms operate without the knowledge of the noise variance or SNR. 

\paragraph{Evaluation Metrics} We evaluate methods in terms of the relative error $\frac{\|\hat{\beta} - \beta^*\|}{\|\beta^\star\|}$, where $\hat{\beta}$ is the last iterate of the applied method, and the negative log likelihood (NLL) for the symmetric 2-component MLR problem \eqref{Class_P}. Note that NLL can be computed without knowledge of the true underlying regressor $\beta^\star$ and noise variance $\sigma^2$.

\paragraph{Baselines} We compare WMLR against EM and Gradient EM (GEM), which is similar to EM, but instead of solving the maximization problem at each iteration, it takes one gradient ascent step. The noise variance for EM and GEM is initialized as ${\sigma^2}^{(0)} = 1$. See Appendix \ref{sec:app_EM_MLR} for additional details and discussion of the EM and GEM algorithms for two-component MLR. We do not compare these algorithms to GAN based methods in this work since GAN-based methods usually take thousands of iterations to converge, as shown by \citet{farnia2020gat} for Gaussian mixture models.

\subsection{Centralized Setting}
For all experiments, the initial iterates $\beta^{(0)}$, $\gamma_1^{(0)}$ and $\gamma_2^{(0)}$ are all chosen i.i.d. from $\normal(0,\, \frac{1}{{d}}I)$. Note that these initializations will have approximately unit norm \cite{vershynin2018high}.
We use the eigenvector of $\Ex{y^2xx^T}$ associated with the largest eigenvalue as the reference vector $\gammatilde$; however, the algorithm is not sensitive to this parameter. WMLR simply needs a reference vector that has non-negligible correlation with $\beta^\star$ to avoid vanishing gradients (also see Theorem \ref{theorem:equal_to_EM}). 

We compare the solution reached at iteration 100 of each algorithm in Table \ref{tab:centralized_exp}. We evaluate each algorithm over several hyperparameter choices, and we choose the run with the smallest final negative log likelihood.
Both WMLR and EM converge quickly (under 100 iterations) while GEM often does not converge by that number of iterations, as seen in the higher SNR case. In the low SNR case, all three algorithms have similar performance.
In the high SNR case, WMLR outperforms EM and GEM both in terms of negative log likelihood and the distance to the true parameter. However, one drawback of WMLR and GEM compared to EM is that WMLR and GEM require hyperparameter tuning.
For additional discussion, implementation details, and the hyperparameter selection, see Appendix \ref{sec:app_numerical}.

\begin{table}[t]
    \centering
        \caption{Comparison of algorithms at iteration $T = 100$ ($\beta^{(T)})$ in terms of negative log likelihood (NLL) and relative $\ell_2$ error, $\|\beta^{(T)} - \beta^\star\|/\|\beta^\star\|$.}
    \label{tab:centralized_exp}
     \vspace{3mm}
    \begin{tabular}{|c|c||l|l|}
    \hline
    \multicolumn{4}{|c|}{Centralized Experiments, $n = 100,000$}\\
    \hline
    { SNR } & Method &  {\hspace{0.5cm}NLL\hspace{0.5cm}} & {Relative $\ell_2$ error}\\
    \hline
    \hline
    &EM             &${2.115}$         &${3.79\times 10^{-2}}$\\
    10&GEM          &${3.765}$         &${1.03}$ \\
    &WMLR           &$\mathbf{2.059}$         &${\mathbf{5.31\times 10^{-3}}}$\\
    \hline
    &EM             &${1.657}$        &${8.62\times 10^{-2}}$\\
    1&GEM           &$\mathbf{1.656}$        &$\mathbf{5.20\times 10^{-2}}$\\
    &WMLR           &$\mathbf{1.656}$        &${7.78\times 10^{-2}}$\\
    \hline
    \hline
    \multicolumn{4}{|c|}{Centralized Experiments, $n = 10,000$}\\
    \hline
    \hline
    &EM             &${2.715}$          &${1.21\times 10^{-1}}$\\
    10&GEM          &${3.758}$          &${9.98 \times 10^{-1}}$\\
    &WMLR           &$\mathbf{2.065}$   &${\mathbf{2.08\times 10^{-2}}}$\\
    \hline
    &EM             &${1.671}$          &${2.95\times 10^{-1}}$\\
    1&GEM           &$\mathbf{1.657}$   &$\mathbf{1.80 \times 10^{-1}}$\\
    &WMLR           &${1.668}$          &${2.75\times 10^{-1}}$\\
    \hline

    \end{tabular}
\end{table}

\subsection{Federated Setting}
As described in Algorithm \ref{alg:F-WMLR}, we extend WMLR to F-WMLR by broadcasting the model to all agents from the central node at each iteration, having each agent take one gradient decent ascent step using his or her own data, and then averaging the resulting new iterates at the central node.

Recall that the EM algorithm operates via two repeated steps: an expectation step and a \textit{full} maximization step. However, in the federated setting, we cannot expect the average of the maximizers to be the maximizer of the average. 
Here, we implement EM in the following way: For each maximization, we perform several communication rounds to solve the maximization problem at each EM step via gradient ascent. We stop this inner maximization when the norm of the gradient is under the threshold $\nu = 0.01$ or after 50 iterations.

We simulate $M = 10,000$ agents with $10$ data points each. 
We assume that each agent $m \in \{1, ..., M\}$ has all her samples drawn from only one of of the two regressors, \ie agent $m$'s samples $(y_{m,i}, x_{m,i})_{i=1}^n$ satisfy
\begin{equation}\label{eq:fl_model}
    y_{m,i} =  z_m(\beta^*)^\top x_{m,i} + \epsilon_{m,i}, \quad i = 1, ..., n,
\end{equation}
where $z_m$ is drawn from $\mathrm{Unif}\left(\{-1, 1\}\right)$.
Again, we draw $\beta^{(0)}, \gamma_1^{(0)}, \gamma_2^{(0)}$ from $\normal(0, \frac{1}{d}I)$. 
The final solutions and the convergence behaviors of these algorithms are compared in Table \ref{tab:fl_exp} and Figure \ref{fig:FL_plot}. EM does not converge in 10,000 iterations for the medium and high SNR cases. 
In our experiments, GEM and WMLR both converged to a comparable level of relative error (Table 2).
Theoretically, both of these methods should converge to the optimal $\beta^\star$ in the population case, so the observed error is mainly due to their generalization performance. Although WMLR takes longer per iteration (about 3min for WMLR vs 1min for GEM in the federated case), WMLR is overall much faster due to the small number of iterations.  WMLR consistently converges in 60 to 100 iterations regardless of SNR, whereas GEM is fast in the low SNR cases but is up to 175x slower in the highest SNR case.

In addition, there is a significant communication cost in the federated setting. Therefore, WMLR's smaller iteration number is particularly important in this setting. While WMLR's implementation is more complex, WMLR enjoys higher robustness to the choice of hyperparameters than GEM. The same hyperparamaters work for all tested SNRs (Figure 2 and Table 3 in the appendix), and iteration count is comparable across all SNRs. Since communication rounds are very costly in the federated learning setting, these results suggest that WMLR may be better equipped than GEM or EM to handle distributed multimodal learning tasks. Additional details are provided in Appendix \ref{sec:app_numerical}.

\begin{table}[h]
    \centering
    \caption{Comparison of algorithms at the final iterate in terms of relative $\ell_2$ error, $\|\beta^{(T)} - \beta^\star\|/\|\beta^\star\|$. The iterations required for convergence is also compared. Note that EM did not convergence (d.n.c.) for SNR = 10 and SNR = 5 cases within 5,000 iterations.}
    \label{tab:fl_exp}
    \vspace{3mm}
    \begin{tabular}{|c|c||l|l|}
    \hline
    \multicolumn{4}{|c|}{Federated Experiments, Final Iterate}\\
    \hline
    { SNR } & Method &  {Iterations Req.} & {Relative $\ell_2$ error}\\
    \hline
    \hline
    &EM         &d.n.c& d.n.c \\
    20&GEM      &12,948& $\mathbf{1.93\times10^{-3}}$ \\
    &WMLR       &$\mathbf{74}$&$2.49\times10^{-3}$ \\
    \hline
    &EM         &d.n.c& d.n.c \\
    10&GEM      &$2,007$& $\mathbf{3.92\times10^{-3}}$ \\
    &WMLR       &$\mathbf{98}$&$4.93\times10^{-3}$ \\
    \hline
    &EM         &d.n.c& d.n.c \\
    5&GEM      &$295$& $\mathbf{8.32\times10^{-3}}$ \\
    &WMLR       &$\mathbf{81}$&$9.95\times10^{-3}$ \\
    \hline
    &EM         &$544$ & $5.60\times10^{-2}$ \\
    1&GEM       &$\mathbf{15}$ & $\mathbf{5.60\times10^{-2}}$\\
    &WMLR       &$66$&$7.25\times10^{-2}$ \\
    \hline
    \end{tabular}
\end{table}

\begin{figure}[t]
    \centering
    \includegraphics[width=.99\linewidth]{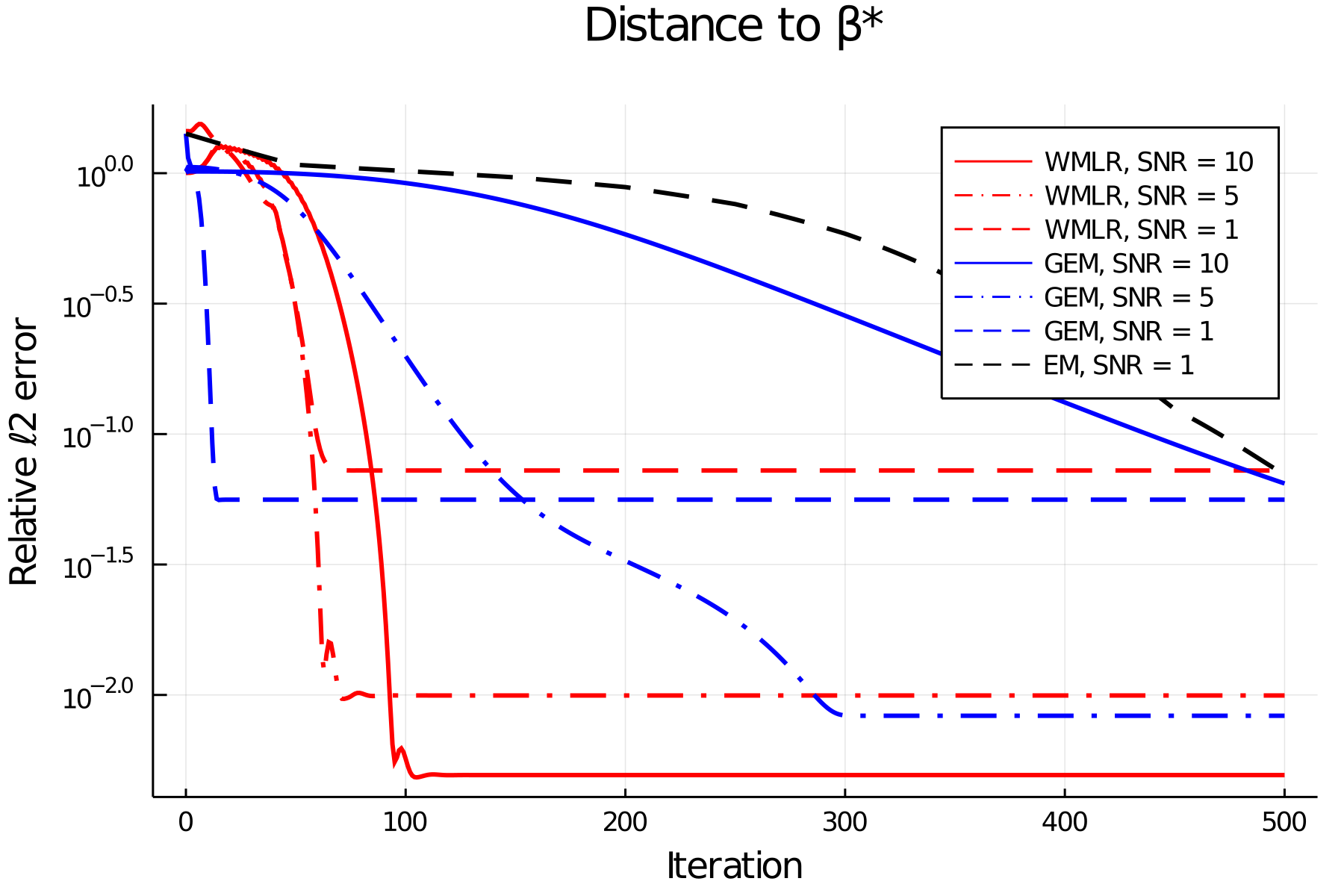}
    \caption{Convergence of $\hat{\beta}$ to $\beta^\star$ in the federated setting with 10,000 nodes with 10 samples each. EM is removed for tests which it did not converge to a reasonable value within 5,000 iterations.}
    \label{fig:FL_plot}
\end{figure}

\section{Acknowledgment}
This paper was partially funded by MIT-IBM Watson AI Lab and Defense Science and Technology Agency. This work was also supported by the QuantERA grant C’MON-QSENS!.
Alireza Fallah acknowledges support from MathWorks Engineering Fellowship.

%%%%%%%%
\bibliography{refs}
\bibliographystyle{icml2021}
%%%%%%%%
\onecolumn
\newpage
\appendix

\begin{center}
\textbf{\LARGE{Appendix}}
\end{center}

\section{Proof of Theorem \ref{theorem:bound_error}}\label{proof-theorem:bound_error}
In order to prove this theorem, note that Theorem 1 in \cite{farnia2020gat} shows that for a fixed $x$ we have
\begin{align*}
    0\, &\le\, W_2^2(P_{\beta_{[k]}}(y|x),P_{\beta_{[k]}^*}(y|x)) 
    - \mathbb{E}_{p_{\beta_{[k]}^*}}[\tilde{\psi}(x,y)] -  \mathbb{E}_{p_{\beta_{[k]}}}[\tilde{\psi}^c(x,y)] \\
    &\le\, \bigl(8\sqrt{\mathbb{E}[y^4|x]} +8\mathbb{E}[y^2|x]\bigr) \sqrt{P_{\operatorname{err}}(x)} + 2\mathbb{E}[y^2|x] \sqrt[4]{P_{\operatorname{err}}(x)}. 
\end{align*}
Also, note that the multiplicative matrix $\Gamma_i$'s in \cite{farnia2020gat}'s Theorem 1 will be equal to the identity matrix based on the theorem's assumptions. Since we assume $\vert (\beta_i^*)^\top x \vert\le C'$ holds with probability $1$, we have
\begin{equation*}
 \mathbb{E}[y^2|x] \le C'^2+\sigma^2,\quad \mathbb{E}[y^4|x] \le C'^4+3\sigma^4+ 6C'^2\sigma^2.
\end{equation*}
Therefore, we obtain the following inequalities
\begin{align*}
    0\, &\le\, W_2^2(P_{\beta}(y|x),P_{\beta^*}(y|x)) - 
    - \mathbb{E}_{p_{\beta_{[k]}^*}}[\tilde{\psi}(x,y)] -  \mathbb{E}_{p_{\beta_{[k]}}}[\tilde{\psi}^c(x,y)] \\
    &\le\, \bigl(8\sqrt{C'^4+3\sigma^4+6C'^2\sigma^2} +8(C'^2+\sigma^2)\bigr) \sqrt{P_{\operatorname{err}}(x)} + 2(C'^2+\sigma^2) \sqrt[4]{P_{\operatorname{err}}(x)}\\
    &\le\, 16(C'^2+2\sigma^2)  \sqrt{P_{\operatorname{err}}(x)} + 2(C'^2+\sigma^2) \sqrt[4]{P_{\operatorname{err}}(x)}.
\end{align*}
Furthermore, note that $\sqrt{p}$ and $\sqrt[4]{p}$ are both concave functions, and hance an application of Jensen's inequality implies the following result since $P_{\operatorname{err}}= \mathbb{E}[P_{\operatorname{err}}(x)]$:
\begin{align*}
    0\, &\le\, \mathbb{E}_{P_X}\bigl[W_c(P_{\beta}(y|x),P_{\beta^*}(y|x))\bigr] 
    - \mathbb{E}_{p_{\beta_{[k]}^*}}[\tilde{\psi}(x,y)] -  \mathbb{E}_{p_{\beta_{[k]}}}[\tilde{\psi}^c(x,y)] \\
    &\le\, 16(C'^2+2\sigma^2) \sqrt{P_{\operatorname{err}}} + 2(C'^2+\sigma^2) \sqrt[4]{P_{\operatorname{err}}}. 
\end{align*}
Therefore, the proof is complete.

%%%%%%%%%%%%%%%%%%%%%%%%%%%%%%%%%%%%%%%
%%%%%%%%%%%%%%%%%%%%%%%%%%%%%%%%%%%%%%%
%%%%%%%%%%%%%%%%%%%%%%%%%%%%%%%%%%%%%%%
\section{Proof of Proposition \ref{proposition:approximation}}
\label{proof-proposition:approximation}
Consider the function $\tilde{\Psi}$ and note that it can be written as follows:
\begin{align}
  \tilde{\Psi}(x,y) &=  y+\sum_{i=1}^k \Pr {Z=i|x=x,Y=y}(\beta_i^*-\beta_i)^\top x \\
  &= y+ \sum_{i=1}^k \Pr {Z=i|x=x,Y=y}{\beta_i^*}^\top x - \sum_{i=1}^k \Pr {Z=i|x=x,Y=y}\beta_i^\top x.
\end{align}
Here, we define 
\begin{equation}
    \Phi_{(\beta_i)_{i=1}^k}(x,y):= \log\bigl( \sum_{i=1}^k \exp(\beta_i^\top x y)\bigr).
\end{equation}
Then, we have
\begin{equation}
    \frac{\partial \Phi_{(\beta^*_i)_{i=1}^k}}{\partial y}(x,y):= \sum_{i=1}^k \operatorname{Pr}_{(\beta_i)_{i=1}^k} ({Z=i|x=x,Y=y}){\beta_i^*}^\top x.
\end{equation}
Therefore,
\begin{align}
    \tilde{\Psi} &= y + \frac{\partial \Phi_{(\beta^*_i)_{i=1}^k}}{\partial y}(x,y) - \frac{\partial \Phi_{(\beta_i)_{i=1}^k}}{\partial y}(x,y) \nonumber \\
    \quad & + \sum_{i=1}^k \biggl[\operatorname{Pr}_{(\beta_i)_{i=1}^k}(Z=i|x=x,Y=y)-\operatorname{Pr}_{(\beta^*_i)_{i=1}^k}(Z=i|x=x,Y=y)\biggr]\beta_i^\top x.
\end{align}
Therefore, under the proposition's assumptions
\begin{equation}
    \bigl\vert \, \tilde{\Psi} - y - \frac{\partial \big\{\Phi_{(\beta^*_i)_{i=1}^k} - \Phi_{(\beta_i)_{i=1}^k} \big\}}{\partial y}(x,y) \,  \bigl\vert \le C'\delta. 
\end{equation}
Note that according to the definitions we have
\begin{equation}
    \Phi_{(\beta^*_i)_{i=1}^k}(x,y)- \Phi_{(\beta_i)_{i=1}^k}(x,y)= \log\biggl( \frac{\sum_{i=1}^k \exp(\beta_i^\top x y)}{\sum_{i=1}^k \exp({\beta^*_i}^\top x y)}\biggr).
\end{equation}
The above two equations complete the proposition's proof.

%%%%%%%%%%%%%%%%%%%%%%%%%%%%%%%%%%%%%%%
%%%%%%%%%%%%%%%%%%%%%%%%%%%%%%%%%%%%%%%
%%%%%%%%%%%%%%%%%%%%%%%%%%%%%%%%%%%%%%%

\section{Proof of Proposition \ref{prop:bounding_c_transform}}\label{proof-prop:bounding_c_transform}
Note that due to the tower property of conditional expectation we have:
\begin{align*}
    &\Ex{\psi^c_{\gamma_{[2k]}}({x},y)}\\
    \stackrel{(a)}{=}  &\mathbb{E}\bigl[\mathbb{E}[\psi^c_{\gamma_{[2k]}}({x},y)|x] \bigr] \\
    \stackrel{(b)}{\le} &\mathbb{E}\biggl[\psi_{\gamma_{[2k]}}({x},y)+\mathbb{E}\biggl[\frac{3k^2\Vert x\Vert^2_2\mathbb{E}[y^2|x]}{1-\eta}\sum_{j=1}^k \bigl[ \Vert\gamma_j^Tx - \tilde{\gamma}_j^Tx\Vert_2^2+\Vert{\gamma}_{j+k}^Tx - \tilde{\gamma}_{j}^Tx\Vert_2^2 + \vert{\gamma}_{j+k}  -{\gamma}_{j}\vert^2 \Vert x\Vert^4_2\bigr]|x\biggr]\biggr] \\
    \stackrel{(c)}{\le}&\mathbb{E}\bigl[\psi_{\gamma_{[2k]}}({x},y)\bigr]+ \mathbb{E}\biggl[\mathbb{E}\biggl[\frac{3k^2\Vert x\Vert^4_2\mathbb{E}[y^2|x]}{1-\eta}\sum_{j=1}^k \bigl[ \Vert\gamma_j - \tilde{\gamma}_j\Vert_2^2+\Vert{\gamma}_{j+k} - \tilde{\gamma}_{j}\Vert_2^2 + \Vert{\gamma}_{j+k}  -{\gamma}_{j}\Vert^2\Vert x\Vert^2_2 \bigr] | x \biggr]\biggr] \\
    \stackrel{(d)}{\le} &\mathbb{E}\bigl[\psi_{\gamma_{[2k]}}({x},y)\bigr]+\mathbb{E}\biggl[\frac{3k^2C^4(\sigma^2+\eta^2)}{1-\eta}\sum_{j=1}^k \bigl[ \Vert\gamma_j - \tilde{\gamma}_j\Vert_2^2+\Vert{\gamma}_{j+k} - \tilde{\gamma}_{j}\Vert_2^2 + C^2\Vert{\gamma}_{j+k}  -{\gamma}_{j}\Vert_2^2 \bigr]\biggr]
    \\
    \stackrel{(e)}{\le} &\mathbb{E}\bigl[\psi_{\gamma_{[2k]}}({x},y)\bigr]+\frac{6k^2C^4(1+C^2)^2}{1-\eta}\sum_{j=1}^k \bigl[ \Vert\gamma_j - \tilde{\gamma}_j\Vert_2^2+\Vert{\gamma}_{j+k} - \tilde{\gamma}_{j}\Vert_2^2 \bigr].
\end{align*}
In the above, (a) follows from the tower property of conditional expectation. (b) is a consequence of {\cite{farnia2020gat}}, Proposition 2 for reference vectors $\tilde{\gamma}^T_jx$. (c) comes from the application of the Cauchy–Schwarz inequality. (d) uses the bounded norm of $x$, and (e) follows from the assumption $\eta<1$ and the application of Young's inequality implying that
\begin{equation}
    \Vert{\gamma}_{j+k}  -{\gamma}_{j}\Vert_2^2 \le 2 \bigl( \Vert{\gamma}_{j}  -\tilde{\gamma}_{j}\Vert_2^2+ \Vert{\gamma}_{j+k}  -\tilde{\gamma}_{j}\Vert_2^2\bigr). 
\end{equation}
Therefore, the proof is complete.

\section{Proof of Theorem \ref{theorem:convergence_to_stationary_pt}}
\label{proof-theorem:convergence_to_stationary_pt}
This theorem follows from the convergence results of Theorem 4.4 in Lin et al. \cite{lin2019gradient}, restated in Lemma \ref{lemma:GDA_convergence} below. 

\begin{lemma} \label{lemma:GDA_convergence}
(Theorem 4.4 in \cite{lin2019gradient})
Consider a $L$-smooth function $f(\beta,\gamma)$ where $f(\beta, \cdot)$ is $\mu$-strongly concave with $\gamma \in \Gamma$, a convex set with diameter $D$. Define condition number $\kappa = L/\mu$,
\begin{equation}
    \Phi(\cdot) = \max_{\gamma \in \Gamma} f(\cdot, \gamma),
\end{equation}
and $\Delta = \Phi(\beta^{(0)}) - \min_\beta\Phi(\beta)$. Then GDA returns an $\epsilon$-stationary point in $\mathcal{O}\left(\frac{\kappa^2L\Delta + \kappa L^2D^2}{\epsilon^2}\right)$ iterations when step sizes are chosen to be $\eta_\beta = \Theta(1/\kappa^2L)$ and $\eta_\gamma = \Theta(1/L)$.
\end{lemma}

First consider the inner maximization problem. $\psi(\bfx, y)$ is the difference of two log-sum-exp terms. Since log-sum-exp has a Hessian with maximum eigenvalue bounded by 1, the norm of the Hessian of the non-concave terms (with respect to $\gamma_i$) are bounded by $\Ex{\|y\bfx\|^2} \leq \eta$. Thus, the inner maximization is $\lambda - 2\eta$ strongly concave. Furthermore, the inner-maximization is $\lambda + 4\eta$ smooth with respect to the vector of maximization variables. 

The gradient of the objective with respect to $\beta$ is $2(\|\gamma_1\| + \|\gamma_2\|)\|y\bfx\|^2$-Lipschitz. Note that the optimal $\gamma_i$'s will be no further than $\eta/\lambda$ away from the reference vector $\gammatilde$. As a result, the optimal $\gamma_i$'s satisfy
\begin{align}
    \|\gamma_i\| \leq \|\gamma_i - \gammatilde\| + \|\gammatilde\| \leq \frac{\eta}{\lambda} + \|\gammatilde\|.
\end{align}
Thus, the objective is $4\eta(\eta/\lambda + \|\gammatilde\|)$ smooth with respect to $\beta$. Applying 
Theorem 4.4 in \cite{lin2019gradient}
completes the proof.
%%%%%%%%%%%%%%%%%%%%%%%%%%%%%%%%%%%%%%%
%%%%%%%%%%%%%%%%%%%%%%%%%%%%%%%%%%%%%%%
%%%%%%%%%%%%%%%%%%%%%%%%%%%%%%%%%%%%%%%
\section{Proof of Theorem \ref{theorem:equal_to_EM}}
\label{proof-theorem:equal_to_EM}
In order to prove Theorem 3, note that at a stationary minimax point $(\hat{\beta},\hat{\gamma}_1,\hat{\gamma}_2)$ we will have:
\begin{align*}
    \nabla_\beta \mathbb{E}\bigl[\log\bigl( \frac{\exp(y\hat{\gamma}_1^Tx) +\exp(-y\hat{\gamma}_1^Tx)}{\exp(y\hat{\gamma}_2^Tx) +\exp(-y\hat{\gamma}_1^Tx)}\bigr) \bigr] &= \mathbf{0}.
\end{align*}
\textbf{Claim:} Under the theorem's assumptions, the above equation implies that $\hat{\gamma}_1=\hat{\gamma}_2$ .

To show this claim, note that
\begin{equation*}
   \nabla_\beta \mathbb{E}\bigl[\log\bigl( \frac{\exp(y\hat{\gamma}_1^Tx) +\exp(-y\hat{\gamma}_1^Tx)}{\exp(y\hat{\gamma}_2^Tx) +\exp(-y\hat{\gamma}_1^Tx)}\bigr) \bigr] =  \mathbb{E}\bigl[\tanh{(y\hat{\gamma}_1^Tx)}xx^T \bigr]\hat{\gamma}_1 - \mathbb{E}\bigl[\tanh{(y\hat{\gamma}_2^Tx)}xx^T \bigr]\hat{\gamma}_2.
\end{equation*}
As a result, the optimality condition implies that
\begin{equation}
    \mathbb{E}\bigl[\tanh(y\hat{\gamma}_1^Tx)xx^T \bigr]\hat{\gamma}_1 = \mathbb{E}\bigl[\tanh(y\hat{\gamma}_2^Tx)xx^T \bigr]\hat{\gamma}_2.
\end{equation}
In addition, the following will be the partial derivative of the above expression with respect to $\gamma$:
\begin{equation}
    \frac{\partial}{\partial \gamma}\mathbb{E}\bigl[\tanh(y{\gamma}^Tx)xx^T\gamma \bigr] = \mathbb{E}\bigl[h(y{\gamma}^Tx)xx^T\bigr],
\end{equation}
where we define $h(z):= \tanh(z)+z\tanh'(z)$ which is an odd increasing function. Note that due to the theorem's assumption for the optimal solution we have $\gamma^Txx^T\beta > 0$ (or the reverse inequality) to always hold. Therefore, without loss of generality we can assume $\mathbb{E}[y|x]\gamma^Tx> 0$ holds with probability $1$ over $p_x$. We claim that this result implies that
\begin{equation}
    \mathbb{E}\bigl[ h(y\gamma^T x)|x \bigr]> 0.
\end{equation}
The above equality holds because $h$ is an odd increasing function, and $\mathbb{E}[y|x]\gamma^T x=0$ results in the following:
\begin{equation*}
 \mathbb{E}[h(y\gamma^T x)|x]=0.   
\end{equation*}
As a result, assuming $\mathbb{E}[y|x]\gamma^T x> 0$ the following inequality holds with probability $1$ over $p_x$, because $\mathbb{E}[h(Z)]$ is increasing in the mean of $Z$ for a normally-distributed $Z$ with a fixed variance: 
\begin{equation*}
 \mathbb{E}[h(y\gamma^T x)|x]\ge 0.   
\end{equation*}
Applying the tower property of conditional expectation completes the claim's proof:
\begin{align*}
 \frac{\partial}{\partial \gamma}\mathbb{E}\bigl[\tanh(y{\gamma}^Tx)xx^T\gamma \bigr] &= \mathbb{E}\bigl[h(y{\gamma}^Tx)xx^T\bigr], \\
 &= \mathbb{E}\bigl[\mathbb{E}[ h(y{\gamma}^Tx)xx^T|x]\bigr]\\
 &= \mathbb{E}\bigl[\mathbb{E}[ h(y{\gamma}^Tx)|x]xx^T\bigr]\\
 & \succ \mathbf{0}.
\end{align*}
Therefore, the claim holds since we showed for the feasible $\gamma$'s $\mathbb{E}\bigl[\tanh(y{\gamma}^Tx)xx^T\gamma \bigr]$ will provide an invertible mapping for $\gamma$.

Showing that $\hat{\gamma}_1=\hat{\gamma}_2$, we further claim that $\hat{\gamma}_1=\hat{\gamma}_2=\tilde{\gamma}$ since otherwise the maximization objective will be $-\Vert \hat{\gamma}_1 - \tilde{\gamma}\Vert^2 -\Vert \hat{\gamma}_2 - \tilde{\gamma}\Vert^2$ which is not optimal given that $\hat{\gamma}_1=\hat{\gamma}_2=\tilde{\gamma}$ achieves a larger value of $0$. Consequently, the optimality condition for the maximization problem at $\hat{\gamma}_1=\hat{\gamma}_2=\tilde{\gamma}$ shows that
\begin{equation}
    \mathbb{E}_{p(\hat{\beta})}\bigl[yx\tanh(y\tilde{\gamma}^Tx) \bigr] = \mathbb{E}_{p(\beta^*)}\bigl[yx\tanh(y\tilde{\gamma}^Tx) \bigr].
\end{equation}
We claim that the above equality implies that either $\hat{\beta}=\beta^*$ or $\hat{\beta}=-\beta^*$ holds. To see this, note that
\begin{equation}
    \frac{\partial}{\partial \beta}\mathbb{E}_{p(\beta)}\bigl[yx\tanh(y\tilde{\gamma}^Tx) \bigr] = \mathbb{E}_{p(\beta)}\bigl[h(y\tilde{\gamma}^Tx) xx^T\bigr]
\end{equation}
where $h(z):= \tanh(z)+z\tanh'(z)$ is the previously defined odd and increasing function. As we showed earlier, the assumption that $\tilde{\gamma}^Txx^T\beta> 0$ holds with probability $1$ implies that the above partial derivative is positive definite:
\begin{equation}
     \mathbb{E}_{p(\beta)}\bigl[h(y\tilde{\gamma}^Tx) xx^T\bigr] \succ \mathbf{0}.  
\end{equation}
As a result either $\{\hat{\beta}, \beta^*\}$ or $\{\hat{\beta}, -\beta^*\}$ is a subset of a set with an invertible $ \mathbb{E}_{p({\beta})}\bigl[yx\tanh(y\tilde{\gamma}^Tx) \bigr]$ mapping from $\beta$. As a result, we have either $\hat{\beta}=\beta^* $ or $\hat{\beta}=-\beta^*$, which completes the proof.
%%%%%%%%%%%%%%%%%%%%%%%%%%%%%%%%%%%%%%%
%%%%%%%%%%%%%%%%%%%%%%%%%%%%%%%%%%%%%%%
%%%%%%%%%%%%%%%%%%%%%%%%%%%%%%%%%%%%%%%

\section{Proof of Theorem 4}
To show this result note that
\begin{align}
\begin{split}
    \mathcal{L}(\beta) := \max_{\gamma_1,\gamma_2}\: &\Ep {\log\left(\frac{\exp(-y\gamma^\top _1\mathbf{x})+\exp(y\gamma^\top _1\mathbf{x})}{\exp(-y\gamma^\top _2\mathbf{x})+\exp(y\gamma^\top _2\mathbf{x})}\right)} -\Ep{\log\left(\frac{\exp(-y\gamma^\top _1\mathbf{x})+\exp(y\gamma^\top _1\mathbf{x})}{\exp(-y\gamma^\top _2\mathbf{x})+\exp(y\gamma^\top _2\mathbf{x})}\right)} \\
    &\qquad\qquad- \frac{\lambda}{2}\left(\|\gamma_1 - \gammatilde\|^2 + \|\gamma_2 - \gammatilde\|^2\right), \\ \widehat{\mathcal{L}}(\beta) := \max_{\gamma_1,\gamma_2}\: &\Ephat{\log\left(\frac{\exp(-y\gamma^\top _1\mathbf{x})+\exp(y\gamma^\top _1\mathbf{x})}{\exp(-y\gamma^\top _2\mathbf{x})+\exp(y\gamma^\top _2\mathbf{x})}\right)} -\Ep{\log\left(\frac{\exp(-y\gamma^\top _1\mathbf{x})+\exp(y\gamma^\top _1\mathbf{x})}{\exp(-y\gamma^\top _2\mathbf{x})+\exp(y\gamma^\top _2\mathbf{x})}\right)} \\
    &\qquad\qquad- \frac{\lambda}{2}\left(\|\gamma_1 - \gammatilde\|^2 + \|\gamma_2 - \gammatilde\|^2\right).
\end{split}
\end{align}
Therefore, assuming $\gamma^*_1,\gamma^*_2$ are the optimal solutions to the maximization problem for the true distribution and minimization variable $\beta$, and that $\hat{\gamma}^*_1,\hat{\gamma}^*_2$ are the optimal solutions to the maximization problem for the empirical distribution and minimization variable $\beta$ we will have
\begin{equation}
    \mathcal{L}(\beta) - \widehat{\mathcal{L}}({\beta}) \le \Ept{\log\left(\frac{\exp(-y{\gamma^*}^\top _1\mathbf{x})+\exp(y{\gamma^*}^\top _1\mathbf{x})}{\exp(-y{\gamma^*}^\top _2\mathbf{x})+\exp(y{\gamma^*}^\top _2\mathbf{x})}\right)}  - \Ephat {\log\left(\frac{\exp(-y{\gamma^*}^\top _1\mathbf{x})+\exp(y{\gamma^*}^\top _1\mathbf{x})}{\exp(-y{\gamma^*}^\top _2\mathbf{x})+\exp(y{\gamma^*}^\top _2\mathbf{x})}\right)}
\end{equation}
and also
\begin{equation}
    \mathcal{L}(\beta) - \widehat{\mathcal{L}}({\beta}) \ge \Ept{\log\left(\frac{\exp(-y\hat{\gamma^*}^\top _1\mathbf{x})+\exp(y{\gamma^*}^\top _1\mathbf{x})}{\exp(-y\hat{\gamma^*}^\top _2\mathbf{x})+\exp(y{\gamma^*}^\top _2\mathbf{x})}\right)}  - \Ephat {\log\left(\frac{\exp(-y{\gamma^*}^\top _1\mathbf{x})+\exp(y\hat{\gamma^*}^\top _1\mathbf{x})}{\exp(-y{\gamma^*}^\top _2\mathbf{x})+\exp(y\hat{\gamma^*}^\top _2\mathbf{x})}\right)}.
\end{equation}
Also, we have the following bound hold for the norm of the optimal maximization variables:
\begin{equation}
    \max\bigl\{\Vert{\gamma}^*_1-\tilde{\gamma}\Vert,\Vert{\gamma}^*_2-\tilde{\gamma}\Vert,\Vert\hat{\gamma}^*_1-\tilde{\gamma}\Vert,\Vert\hat{\gamma}^*_2-\tilde{\gamma}\Vert \bigr\} \le \frac{C^2\eta+\sigma C}{\lambda} 
\end{equation}
To establish a generalization bound on $\vert \mathcal{L}(\beta) - \widehat{\mathcal{L}}({\beta})\vert$, we bound the following concentration error for every norm-bounded $\Vert \gamma_1 -\tilde{\gamma}\Vert,\Vert \gamma_2 -\tilde{\gamma}\Vert\le \frac{C^2\eta+\sigma C}{\lambda} $: $$\Ephat{\log\bigl(\frac{\exp(-y\gamma_1^\top \mathbf{x})+\exp(y\gamma_1^\top \mathbf{x})}{\exp(-y\gamma_2^\top \mathbf{x})+\exp(y\gamma_2^\top \mathbf{x})}\bigr)}- \Ept{\log\bigl(\frac{\exp(-y\gamma_1^\top \mathbf{x})+\exp(y\gamma_1^\top \mathbf{x})}{\exp(-y\gamma_2^\top \mathbf{x})+\exp(y\gamma_2^\top \mathbf{x})}\bigr)}.$$ 
We claim that $\log\bigl((\exp(-y\gamma^\top \mathbf{x})+\exp(-y\gamma^\top \mathbf{x})\bigr) - \mathbb{E}[\log\bigl((\exp(-y\gamma^\top \mathbf{x})+\exp(-y\gamma^\top \mathbf{x})\bigr)]$ is a sub-Gaussian random variable with degree $C^2\eta^2(C^2\eta^2+\sigma^2)$. This is because
\begin{align*}
    \Pr{ \log(\frac{\exp(-y\gamma^\top \mathbf{x})+\exp(-y\gamma^\top \mathbf{x})}{2}) \ge v } &\le \Pr{ \vert y\gamma^\top \mathbf{x}\vert  \ge v } \\
    &\le \Pr{ \vert y\vert  \ge \frac{v}{\eta C} } \\
    &\le 2\exp\bigl(-\frac{v^2}{C^2\eta^2(C^2\eta^2+\sigma^2)}\bigr).
\end{align*}
The above holds because $y$ is the sum of two independent sub-Gaussian variables, i.e., the bounded $\beta^T\bfx$ and Gaussian $\epsilon$. Therefore, the claim holds and covering all feasible norm-bounded $\gamma_1,\gamma_2$ vectors with $O((\frac{C^2\eta+\sigma C}{\lambda})^d)$ points, we will have the following bound hold with probability at least $1-\delta$ for every norm-bounded $\Vert\beta\Vert_2\le \eta$:
\begin{align*}
    |\mathcal{L}(\beta) - \widehat{\mathcal{L}}(\beta) | &\le O\bigl(\sqrt{\frac{C^2\eta^2(C^2\eta^2+\sigma^2)d\log((C^2\eta+\sigma C)/\lambda\delta)}{n}}\bigr) \\
    &= O\bigl(\sqrt{\frac{C^4\eta^4\sigma^2d\log((C^2\eta+\sigma C)/\lambda\delta)}{n}}\bigr)
\end{align*}
To establish the generalization bound for the objective's gradient, note that the optimal solution to the maximization problem will satisfy the following equations:
\begin{align*}
    \gamma^*_1-\tilde{\gamma}&=\frac{1}{\lambda}\Ept{y\bfx \tanh(\gamma^*_1y\bfx)} - \frac{1}{\lambda}\Ep {y\bfx \tanh(\gamma^*_1y\bfx)}, \\
    \gamma^*_2-\tilde{\gamma}&=\frac{1}{\lambda}\Ep{y\bfx \tanh(\gamma^*_2y\bfx)}-\frac{1}{\lambda}\Ept{y\bfx \tanh(\gamma^*_2y\bfx)}, \\
    \hat{\gamma}^*_1-\tilde{\gamma}&=\frac{1}{\lambda}\Ephat{y\bfx \tanh(\hat{\gamma}^*_1y\bfx)}-\frac{1}{\lambda}\Ep {y\bfx \tanh(\hat{\gamma}^*_1y\bfx)}, \\
    \hat{\gamma}^*_2-\tilde{\gamma}&=\frac{1}{\lambda}\Ep {y\bfx \tanh(\hat{\gamma}^*_2y\bfx)}-\frac{1}{\lambda}\Ephat {y\bfx \tanh(\hat{\gamma}^*_2y\bfx)}.
\end{align*}
Since both $\Vert \bfx \Vert\le C,\, \vert\tanh(\gamma^T y\bfx)\vert\le 1$ are bounded, we have the following tail bound for $y\bfx \tanh(\gamma^T y\bfx)$:
\begin{equation}
\Pr{|y\bfx \tanh(\gamma^T y\bfx)|>v}\le \Pr{|y|>\frac{v}{C}}\le \exp(-\frac{v^2}{C^2(\eta^2 C^2+\sigma^2)}).
\end{equation}
Note that the above holds because $y$ is the sum of two independent sub-Gaussian random variables $\beta^T\bfx$ and $\epsilon$. As a result, $y\bfx \tanh(\gamma^T y\bfx)$ is a sub-Gaussian random variable with degree $C^2(\eta^2 C^2+\sigma^2)$. 
Therefore, for the function $g(\gamma)= \gamma - \frac{1}{\lambda}\Ep{y\bfx\tanh({\gamma}y\bfx)}$ whose Jacobian is $\frac{C\eta}{\lambda}$-close to identity, using a covering for all the potential norm-bounded solution of $\gamma^*$'s we have the following hold with probability at least $\delta$:
\begin{align*}
    \Vert g(\hat{\gamma}^*_1) - g(\gamma^*_1) \Vert &\le O\big(\sqrt{\frac{d C^2(\eta^2 C^2+\sigma^2)\log((C^2\eta+\sigma C)/{\lambda}\delta)}{n}} \bigr) \\
    \Vert g(\hat{\gamma}^*_2) - g({\gamma}^*_2) \Vert &\le O\big(\sqrt{\frac{dC^2(\eta^2 C^2+\sigma^2)\log((C^2\eta+\sigma C)/{\lambda}\delta)}{n}} \bigr)
\end{align*}
which implies that
\begin{align*}
    \Vert \hat{\gamma}^*_2 - {\gamma}^*_2 \Vert &\le O\big(\sqrt{\frac{d C^2(\eta^2 C^2+\sigma^2)\log((C^2\eta+\sigma C)/{\lambda}\delta)}{(1-C\eta/\lambda)^2n}} \bigr)  \bigr) \\
    \Vert \hat{\gamma}^*_2 - {\gamma}^*_2 \Vert &\le O\big(\sqrt{\frac{d C^2(\eta^2 C^2+\sigma^2)\log((C^2\eta+\sigma C)/{\lambda}\delta)}{(1-C\eta/\lambda)^2n}} \bigr) 
\end{align*}
Furthermore, according to the Danskin's theorem we have 
\begin{align*}
    \nabla \mathcal{L}(\beta) &=  -\nabla_\beta\Ep{\log\left(\frac{\exp(-y{\gamma^*}^\top _1\mathbf{x})+\exp(y{\gamma^*}^\top _1\mathbf{x})}{\exp(-y{\gamma^*}^\top _2\mathbf{x})+\exp(y{\gamma^*}^\top _2\mathbf{x})}\right)} \\
    &= \Ep{\tanh(y{\gamma^*}^\top _2\mathbf{x})\mathbf{x}\mathbf{x}^T{\gamma^*}_2} - \Ep{\tanh(y{\gamma^*}^\top _1\mathbf{x})\mathbf{x}\mathbf{x}^T{\gamma^*}_1}, \\
    \nabla \widehat{\mathcal{L}}(\beta) &=  -\nabla_\beta\Ep{\log\left(\frac{\exp(-y\hat{\gamma^*}^\top _1\mathbf{x})+\exp(y\hat{\gamma^*}^\top _1\mathbf{x})}{\exp(-y\hat{\gamma^*}^\top _2\mathbf{x})+\exp(y\hat{\gamma^*}^\top _2\mathbf{x})}\right)} \\
    &= \Ep{\tanh(y\hat{\gamma^*}^\top _2\mathbf{x})\mathbf{x}\mathbf{x}^T\hat{\gamma^*}_2} - \Ep{\tanh(y\hat{\gamma^*}^\top _1\mathbf{x})\mathbf{x}\mathbf{x}^T{\gamma^*}_1},
\end{align*}
Combining the above two consequences and noting that $h(z)=z\tanh(z)$ is a $1.2$-Lipschiz function show that the following will hold with probability at least $1-\delta$ for every feasible $\beta$
\begin{equation}
    \Vert\nabla\mathcal{L}(\beta)-\nabla\widehat{\mathcal{L}}(\beta)\Vert \le O\big(\sqrt{\frac{d C^{10} \eta^2(\eta^2 C^2+\sigma^2)\log((C^2\eta+\sigma C)/{\lambda}\delta)}{(1-C\eta/\lambda)^2n}} \bigr)  = O\big(\sqrt{\frac{d \log(1/{\lambda}\delta)}{(1-C\eta/\lambda)^2n}} \bigr).
\end{equation}
Therefore, the proof is complete.

\section{The Expectation Maximization Algorithm for Mixed Linear Regression} \label{sec:app_EM_MLR}
EM seeks to find the MLE estimate, \ie the maximizer of the likelihood function.
Because the likelihood function can be computationally expensive to maximize directly, the EM algorithm instead maximizes a lower bound at each step. For an EM tutorial, see \citet{bilmes1998gentle}. For an adaptation to the MLR problem, see \citet{balakrishnan2017statistical} and references therein.

In our setup, the problem data are $(y_i, x_i)_{i=1}^N$, the latent variable is $Z$, and the parameters to be estimated are $\beta$ and $\sigma^2$ (recall the class of distributions in \eqref{Class_P}). Denote the current parameter estimates by $\beta$ and $\sigma^2$ and the next iterates (to be estimated) by $\tilde{\beta}, \tilde{\sigma^2}$. Then the function of interest is
\begin{align}\label{eq:Qfunction}
    Q(\tilde{\beta}, \tilde{\sigma^2} \mid \beta, \sigma^2) &= 
    \frac{1}{N}\sum_{i=1}^N
        \mathbb{P}_{\beta,\sigma^2}\left({Z=1\mid X = x_i, Y = y_i}\right)\log f_{\tilde{\beta}, \tilde{\sigma^2}}(Z=1, X = x_i, Y = y_i)\nonumber \\
        &\qquad\qquad+\mathbb{P}_{\beta,\sigma^2}\left({Z=2\mid X = x_i, Y = y_i}\right)\log f_{\tilde{\beta}, \tilde{\sigma^2}}(Z=2, X = x_i, Y = y_i),
\end{align}
where $f_{\beta, \sigma^2}$ is the likelihood function of $Z, X, Y$, parameterized by $\beta, \sigma^2$. We can simplify \eqref{eq:Qfunction} for the symmetric 2-component MLR case explored in this work. First, define the weight function $w(x_i, y_i)$ as
\begin{equation}\label{eq:weight_function_em}
    w_{\beta,\sigma^2}(x, y) = \mathbb{P}_{\beta,\sigma^2}\left({Z=1\mid X = x_i, Y = y_i}\right) = \frac{\exp(\frac{-1}{2\sigma^2}(y-\beta^\top x)^2)}{\exp(\frac{-1}{2\sigma^2}(y-\beta^\top x)^2) + \exp(\frac{-1}{2\sigma^2}(y+\beta^\top x)^2)}.
\end{equation}
Then, the function $Q(\tilde{\beta}, \tilde{\sigma^2} \mid \beta, \sigma^2)$ can be simplified as
\begin{equation}\label{eq:Qfunction_simp}
    Q(\tilde{\beta}, \tilde{\sigma^2} \mid \beta, \sigma^2) = 
    -\frac{1}{2}\log{\tilde{\sigma^2}} - \frac{1}{2\tilde{\sigma^2} N}\sum_{i=1}^N w_{\beta,\sigma^2}(x_i, y_i) (y - \tilde{\beta}^\top x)^2 + (1-w_{\beta,\sigma^2}(x_i, y_i))(y + \tilde{\beta}^\top x)^2.
\end{equation}

Note that all parameters inside the weight function are fixed over a maximization. When the noise variance $\sigma^2$ is known, there is a closed form solution for the maximizer of \eqref{eq:Qfunction_simp}:
\begin{align*}
    \tilde{\beta} &= \Sigma_X^{-1} \left(\frac{1}{N}\sum_{i=1}^N (2w_{\beta}(x_i, y_i) - 1)y_ix_i \right),\\
    \tilde{\sigma^2} &= \left(\frac{1}{N}\sum_{i=1}^N w_{\beta,\sigma^2}(x_i, y_i) (y - \tilde{\beta}^\top x)^2 + (1-w_{\beta,\sigma^2}(x_i, y_i))(y + \tilde{\beta}^\top x)^2\right)^{-1}
\end{align*}
Recall that we assume that $p_x$ is known, so the covariance of $X$, $\Sigma_X$ is known. In practice, we can estimate $\Sigma_X$ when the number of samples $N$ is high in the centralized setting. In the distributed setting, we must maximize \eqref{eq:Qfunction_simp} iteratively (\eg with gradient ascent). This procedure is summarized in Algorithms \ref{alg:EM} and \ref{alg:F-EM}.
%%%%%%%%%%%%%%%%%%%%%%%%%%%%%%%%%%%%%%%%%%%%%%%%%%%%%%%%%%
%%%%%%%%%%%%%%%%%%%%%%%%%%%%%%%%%%%%%%%%%%%%%%%%%%%%%%%%%%
%%%%%%%%%%%%%%%%%%%%%%%%%%%%%%%%%%%%%%%%%%%%%%%%%%%%%%%%%%
%%%%%%%%%%%%%%%%%%%%%%% Algorithm %%%%%%%%%%%%%%%%%%%%%%%%
%%%%%%%%%%%%%%%%%%%%%%%%%%%%%%%%%%%%%%%%%%%%%%%%%%%%%%%%%%
%%%%%%%%%%%%%%%%%%%%%%%%%%%%%%%%%%%%%%%%%%%%%%%%%%%%%%%%%%
%%%%%%%%%%%%%%%%%%%%%%%%%%%%%%%%%%%%%%%%%%%%%%%%%%%%%%%%%%
\begin{algorithm}[h]
   \caption{EM}
   \label{alg:EM}
\begin{algorithmic}
   \STATE {\bfseries Input:} $(x_{i}, y_{i})_{i \in [n]}$, $\beta^{(0)}$, ${\sigma^2}^{(0)}$.
   \FOR{ $t  = 0$ {\bfseries to} $T-1$ }
   \STATE Compute $\beta^{(t+1)}, {\sigma^2}^{(t+1)} = \argmax_{\beta, \sigma^2} Q(\beta, \sigma^2 \mid \beta^{(t)}, {\sigma^2}^{(t)})$

   \ENDFOR
\end{algorithmic}
\end{algorithm}

\begin{algorithm}[h]
   \caption{F-EM}
   \label{alg:F-EM}
\begin{algorithmic}
   \STATE {\bfseries Input:} $(x_{i,m}, y_{i,m})_{m \in [M],\, i \in [N]}$, $\beta^{(0)}$, ${\sigma^2}^{(0)}$, step size $\alpha$.
    \FOR{ $t  = 0$ { \bfseries to} $T-1$ }
    \STATE $\beta' = \beta^{(t)}$, ${\sigma^2}' = {\sigma^2}^{(t)}$
    \FOR{$i = 1$ { \bfseries to} $50$ }
    \STATE Broadcast $\beta^{(t)}$, ${\sigma^2}^{(t)}$ to all agents
    \FOR{each agent $m=1$ {\bfseries to} $M$}
        
        \STATE $\beta^{(t+1)}_m = \beta^{(t)}  + \alpha  \nabla_\beta Q_m(\beta, \sigma^2 \mid \beta', {\sigma^2}')$
        \STATE ${\sigma^2}^{(t+1)}_m = {\sigma^2}^{(t)} + \alpha \nabla_{\sigma^2}Q_m(\beta, \sigma^2 \mid \beta', {\sigma^2}')$
    \STATE Send $\beta^{(t+1)}_m, {\sigma^2}^{(t+1)}_m$ to server
    \ENDFOR
    \STATE Collect $\beta^{(t+1)}_m, {\sigma^2}^{(t+1)}_m$ from all agents $m \in [M]$
    \STATE $\beta^{(t)}  = \frac{1}{M}\sum_{m=1}^M \beta^{(t+1)}_{m}$
    \STATE ${\sigma^2}^{(t)}  = \frac{1}{M}\sum_{m=1}^M {\sigma^2}^{(t+1)}_{m}$
    
    \IF{$\|\frac{1}{M}\sum_{m=1}^M \nabla_\beta Q_m(\beta, \sigma^2 \mid \beta', {\sigma^2}')\| \leq \nu$} 
        \STATE \textbf{break}
    \ENDIF
    \ENDFOR
    \ENDFOR
\end{algorithmic}
\end{algorithm}
%%%%%%%%%%%%%%%%%%%%%%%%%%%%%%%%%%%%%%%%%%%%%%%%%%%%%%%%%%
%%%%%%%%%%%%%%%%%%%%%%%%%%%%%%%%%%%%%%%%%%%%%%%%%%%%%%%%%%
%%%%%%%%%%%%%%%%%%%%%%%%%%%%%%%%%%%%%%%%%%%%%%%%%%%%%%%%%%

\subsection{Gradient Expectation Maximization Algorithm (GEM)}
Instead of solving the maximization problem entirely at each iteration, GEM takes one gradient ascent step on \eqref{eq:Qfunction_simp}. The gradients are 

\begin{align*}
    \nabla_{\tilde{\beta}} Q(\tilde{\beta}, \tilde{\sigma^2} \mid \beta, \sigma^2) &= \frac{1}{\tilde{\sigma^2} N}\sum_{i=1}^N \left(w_{\beta,\sigma^2}(x_i, y_i) (y - \tilde{\beta}^\top x) + (w_{\beta,\sigma^2}(x_i, y_i) - 1)(y + \tilde{\beta}^\top x)\right)x, \\
    \nabla_{\tilde{\sigma^2}} Q(\tilde{\beta}, \tilde{\sigma^2} \mid \beta, \sigma^2) &= 
    \frac{1}{2\tilde{\sigma^4}N}\sum_{i=1}^N w_{\beta,\sigma^2}(x_i, y_i) (y - \tilde{\beta}^\top x)^2 + (1-w_{\beta,\sigma^2}(x_i, y_i))(y + \tilde{\beta}^\top x)^2 - \frac{1}{2\tilde{\sigma^2}}
\end{align*}

This procedure is outlined in Algorithm \ref{alg:GEM}, and the federated extension in Algorithm \ref{alg:F-GEM}
%%%%%%%%%%%%%%%%%%%%%%%%%%%%%%%%%%%%%%%%%%%%%%%%%%%%%%%%%%
%%%%%%%%%%%%%%%%%%%%%%%%%%%%%%%%%%%%%%%%%%%%%%%%%%%%%%%%%%
%%%%%%%%%%%%%%%%%%%%%%%%%%%%%%%%%%%%%%%%%%%%%%%%%%%%%%%%%%
%%%%%%%%%%%%%%%%%%%%%%% Algorithm %%%%%%%%%%%%%%%%%%%%%%%%

With an appropriate choice of the stepsize $\alpha$, this procedure is an ascent algorithm, \ie 
\begin{equation}
    Q(\beta^{(t+1)}, {\sigma^2}^{(t+1)} \mid \beta^{(t)}, {\sigma^2}^{(t)}) \geq Q(\beta^{(t)}, {\sigma^2}^{(t)} \mid \beta^{(t)}, {\sigma^2}^{(t)}).
\end{equation}
Furthermore, we can incorporate the constraint that $\sigma^2 > 0$ via a projection after the gradient iteration.

%%%%%%%%%%%%%%%%%%%%%%%%%%%%%%%%%%%%%%%%%%%%%%%%%%%%%%%%%%
%%%%%%%%%%%%%%%%%%%%%%%%%%%%%%%%%%%%%%%%%%%%%%%%%%%%%%%%%%
%%%%%%%%%%%%%%%%%%%%%%%%%%%%%%%%%%%%%%%%%%%%%%%%%%%%%%%%%%
\begin{algorithm}[h]
   \caption{Gradient EM}
   \label{alg:GEM}
\begin{algorithmic}
   \STATE {\bfseries Input:} $(x_{i}, y_{i})_{i \in [n]}$, $\beta^{(0)}$, ${\sigma^2}^{(0)}$, step size $\alpha$.
   \FOR{ $t  = 0$ {\bfseries to} $T-1$ }

   \STATE $\beta^{(t+1)} = \beta^{(t+1)}  + \alpha  \nabla_\beta Q(\beta, \sigma^2 \mid \beta^{(t)}, {\sigma^2}^{(t)})$
   \STATE ${\sigma^2}^{(t+1)} = {\sigma^2}^{(t)} + \alpha \nabla_{\sigma^2}Q(\beta, \sigma^2 \mid \beta^{(t)}, {\sigma^2}^{(t)})$

   \ENDFOR
\end{algorithmic}
\end{algorithm}

\begin{algorithm}[h]
   \caption{F-GEM}
   \label{alg:F-GEM}
\begin{algorithmic}
   \STATE {\bfseries Input:} $(x_{i,m}, y_{i,m})_{m \in [M],\, i \in [N]}$, $\beta^{(0)}$, ${\sigma^2}^{(0)}$, step size $\alpha$.

    \FOR{ $t  = 0$ { \bfseries to} $T-1$ }
    \STATE Broadcast $\beta^{(t)}$, ${\sigma^2}^{(t)}$ to all agents
    \FOR{each agent $m=1$ {\bfseries to} $M$}
        \STATE $\beta^{(t+1)}_m = \beta^{(t)}  + \alpha  \nabla_\beta Q_m(\beta, \sigma^2 \mid \beta^{(t)}, {\sigma^2}^{(t)})$
        \STATE ${\sigma^2}^{(t+1)}_m = {\sigma^2}^{(t)} + \alpha \nabla_{\sigma^2}Q_m(\beta, \sigma^2 \mid \beta^{(t)}, {\sigma^2}^{(t)})$
    \STATE Send $\beta^{(t+1)}_m, {\sigma^2}^{(t+1)}_m$ to server
    \ENDFOR
    \STATE Collect $\beta^{(t+1)}_m, {\sigma^2}^{(t+1)}_m$ from all agents $m \in [M]$
    \STATE $\beta^{(t)}  = \frac{1}{M}\sum_{m=1}^M \beta^{(t+1)}_{m}$
    \STATE ${\sigma^2}^{(t)}  = \frac{1}{M}\sum_{m=1}^M {\sigma^2}^{(t+1)}_{m}$
    \ENDFOR
\end{algorithmic}
\end{algorithm}

%%%%%%%%%%%%%%%%%%%%%%%%%%%%%%%%%%%%%%%%%%%%%%%%%%%%%%%%%%
%%%%%%%%%%%%%%%%%%%%%%%%%%%%%%%%%%%%%%%%%%%%%%%%%%%%%%%%%%
%%%%%%%%%%%%%%%%%%%%%%%%%%%%%%%%%%%%%%%%%%%%%%%%%%%%%%%%%%

\section{Numerical Experiments} \label{sec:app_numerical}
All code is included in the supplementary materials. This section provides additional details regarding the implementation. 

\paragraph{Estimating the Noise Variance in WMLR} Recall that we evaluate the negative log likelihood of the estimated regressor. Since WMLR does not estimate $\sigma^2$ explicitly (unlike EM, GEM), we must estimate this from the last iterate $\beta^{(T)} = \hat{\beta}$. We can estimate the noise variance by empirically computing
\begin{align}
    \hat{\sigma^2} = \E{\epsilon^2} = \E{y^2} - \|\hat{\beta}\|^2.
\end{align}
This estimate is used to compute the negative log likelihood for the WMLR algorithm's final iterate.

\paragraph{Motivation for Iteration Count Comparison}
Iteration seems to be (roughly) a good comparison. Gradients of the EM function \eqref{eq:Qfunction_simp} have computational complexity $O(Nd)$, as the dot product is the dominant term in each of the $N$ summands. Similarly, the gradient of $\psi$ with reqspect to $\gamma_1$ and $\gamma_2$ have computational complexity $O(Nd)$ since the terms have the form $\tanh(y\gamma_i^Tx)yx$ (assuming that $\mathrm{tanh}$ can be computed in constant time). The gradient with respect to $\beta$ is a bit trickier, since it is a parameter of a Gaussian distribution from which we generate samples. However, forward-mode AD has complexity that is bounded by a constant factor of the complexity of the function being differentiated \cite{ADreview}.

\paragraph{Hyperparameter Tuning}
GEM has a hyperparameter $\alpha$ that controls the gradient ascent step size (see Algorithm \ref{alg:GEM}). In the centralized experiments, for each SNR, we choose the hyperparameter from 10 points logarithmically spaced between \texttt{1e-4} and \texttt{10} that gives the smallest negative log likelihood. In the federated experiments, we choose the hyperparameter from 20 points logarithmically spaced between \texttt{1e-4} and \texttt{10} that results in the fastest convergence.

WMLR has three hyperparameters: regularization term $\lambda$, maximization step size $\alpha_\mathrm{max}$, and minimization step size $\alpha_\mathrm{min}$. We find the heuristic $\alpha_\mathrm{max} = \frac{1}{2\lambda}$ and $\alpha_\mathrm{min} = \alpha_\mathrm{max}/10$, inspired by Theorem $\ref{theorem:convergence_to_stationary_pt}$, works reasonably well so we only search over $\lambda$. In the centralized experiments, for each SNR, we choose the $\lambda$ from 10 points logarithmically spaced between \texttt{1e-1} and \texttt{2} that gives the smallest negative log likelihood. In the federated experiments, we choose the $\lambda$ from 20 points logarithmically spaced between \texttt{1e-1} and \texttt{2} that results in the fastest convergence. Runs with different hyperparameters are compared in Figure \ref{fig:params}. Values chosen are in Table \ref{tbl:hyperparameters}.

\begin{table}[h]
    \centering
        \caption{Hyperparameter choices for numerical experiments.}
    \label{tbl:hyperparameters}
    \vspace{3mm}
    \begin{tabular}{|c|c|c|}
    \hline
    \multicolumn{3}{|c|}{Federated Experiments, Final Iterate}\\
    \hline
    Method & SNR & Hyperparameter ($\lambda$ or $\alpha$) \\
    \hline
    \multirow{ 2}{*}{EM} 
    &1 & N/A\\
    &10 & N/A\\
    \hline
    \multirow{ 2}{*}{GEM} 
    &1 & $\alpha =2.78$\\
    &10 & did not converge in 100 iterations\\
    \hline
    \multirow{ 2}{*}{WMLR} 
    &1 & $\lambda =0.38$\\
    &10 & $\lambda =0.53$\\
    \hline
    \multirow{ 4}{*}{F-EM} 
    &1 & $\alpha=0.08$\\
    &5 & did not converge\\
    &10 & did not converge\\
    &20 & did not converge\\
    \hline
    \multirow{ 4}{*}{F-GEM} 
    &1 & $\alpha=2.98$\\
    &5 & $\alpha=0.89$\\
    &10 & $\alpha=0.48$\\
    &20 & $\alpha=0.14$\\
    \hline
    \multirow{ 4}{*}{F-WMLR} 
    &1 & $\lambda=0.35$\\
    &5 & $\lambda=0.41$\\
    &10 & $\lambda=0.41$\\
    &20 & $\lambda=0.41$\\
    \hline
    \end{tabular}
\end{table}

\paragraph{Iterations until Convergence} Let $e^{(T)}$ denote the relative $\ell2$ error at the final iterate. We say that an algorithm has converged at iterate $t_0$ if $\forall \, k \in \{t_0, t_0 + 1, ..., T\}$, we have that $e^{(k)} \leq 1.05\cdot e^{(T)}$. In the federated experiments, we list the minimum $t_0$ for which this condition holds.

\paragraph{Confidence Intervals}
We reran the centralized experiments 50 times for WMLR and EM, randomly generating the initialization and data each time. We list confidence intervals in Table \ref{tab:confidence-intervals}. We observe that WMLR consistently outperforms EM as measured by relative $\ell_2$ error.
\begin{table}[h]
    \centering
    \caption{Lower and upper quartiles over 50 runs.}
    \label{tab:confidence-intervals}
    \vspace{3mm}
    \begin{tabular}{|c|c||l|l|}
    \hline
    \multicolumn{4}{|c|}{Centralized Experiments: Confidence Intervals}\\
    \hline
    { SNR } & Method &  {\hspace{0.5cm}NLL\hspace{0.5cm}} & {Relative $\ell_2$ error}\\
    \hline
    \hline
    \multicolumn{4}{|c|}{$n = 100,000$}\\
    \hline
    10&EM             &\small${[2.114,2.126]}$         &\small${[3.68,4.02]\times 10^{-2}}$\\
    &WMLR           &\small${[2.057, 2.105]}$         &\small${{[4.97, 5.64]\times 10^{-3}}}$\\
    \hline
    1&EM             &\small${[1.657, 1.660]}$        &\small${[8.60,9.25]\times 10^{-2}}$\\
    &WMLR           &\small${[1.656, 1.658]}$        &\small${[7.02, 7.58]\times 10^{-2}}$\\
    \hline
    \hline
    \multicolumn{4}{|c|}{$n = 10,000$}\\
    \hline
    10&EM             &\small${[2.729, 2.845]}$         &\small${[1.21, 1.31]\times 10^{-1}}$\\
    &WMLR           &\small${[2.061, 2.223]}$         &\small${{[1.83, 1.99]\times 10^{-2}}}$\\
    \hline
    1&EM             &\small${[1.661, 1.671]}$        &\small${[2.74, 3.05]\times 10^{-1}}$\\
    &WMLR           &\small${[1.655, 1.666]}$        &\small${[2.37, 2.53]\times 10^{-1}}$\\
    \hline
    \end{tabular}
\end{table}

\paragraph{Computing Setup}
All experiments were run on a MacBook Pro with a 2.3GhZ 8-Core Intel i9 processor and 32GB of RAM. For each algorithm, running 100 iterations with 100k samples in the centralized case takes under 2min. Specifically, EM and GEM take about 4 seconds, and WMLR takes about 75 seconds. We suspect that much of this difference may come from implementation (\eg we use automatic differntiation for WMLR, whereas we wrote an efficient, non-allocation gradient function for GEM and we analytically compute the maximizer in EM). In the federated case, all methods took on the order of 1-3min per 100 iterations.

\begin{figure}
    \centering
    \includegraphics[width=0.32\linewidth]{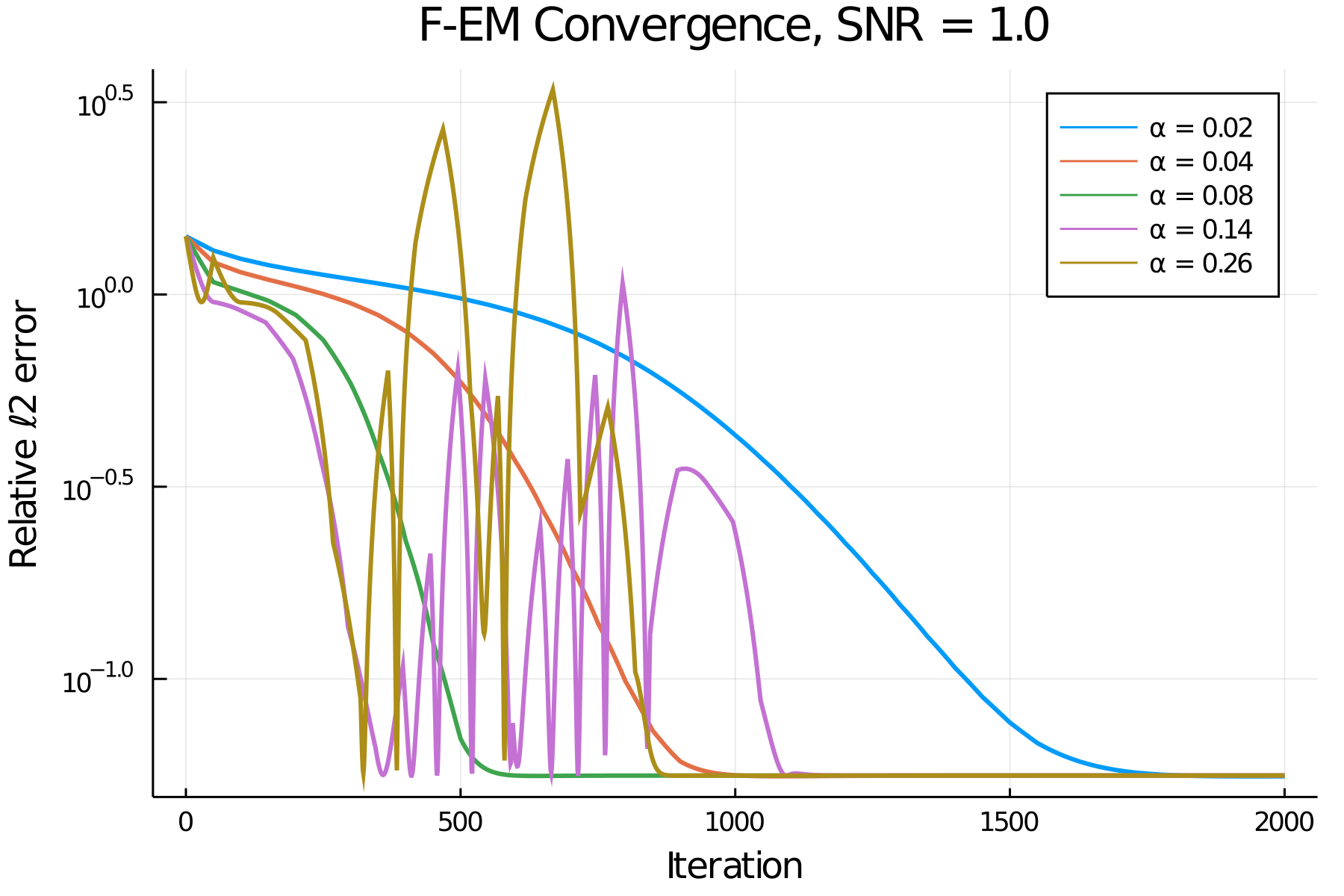}
    \includegraphics[width=0.32\linewidth]{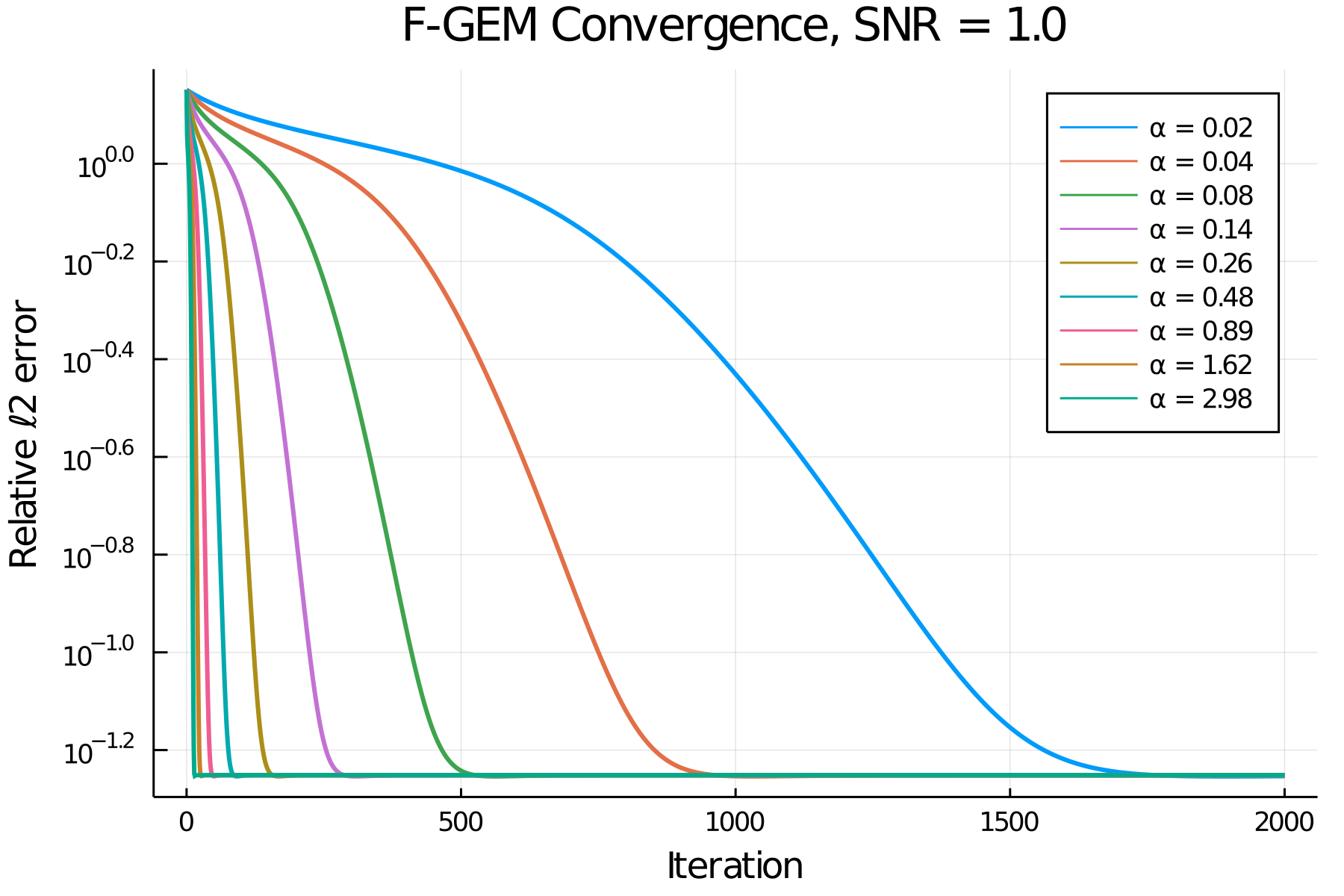}
    \includegraphics[width=0.32\linewidth]{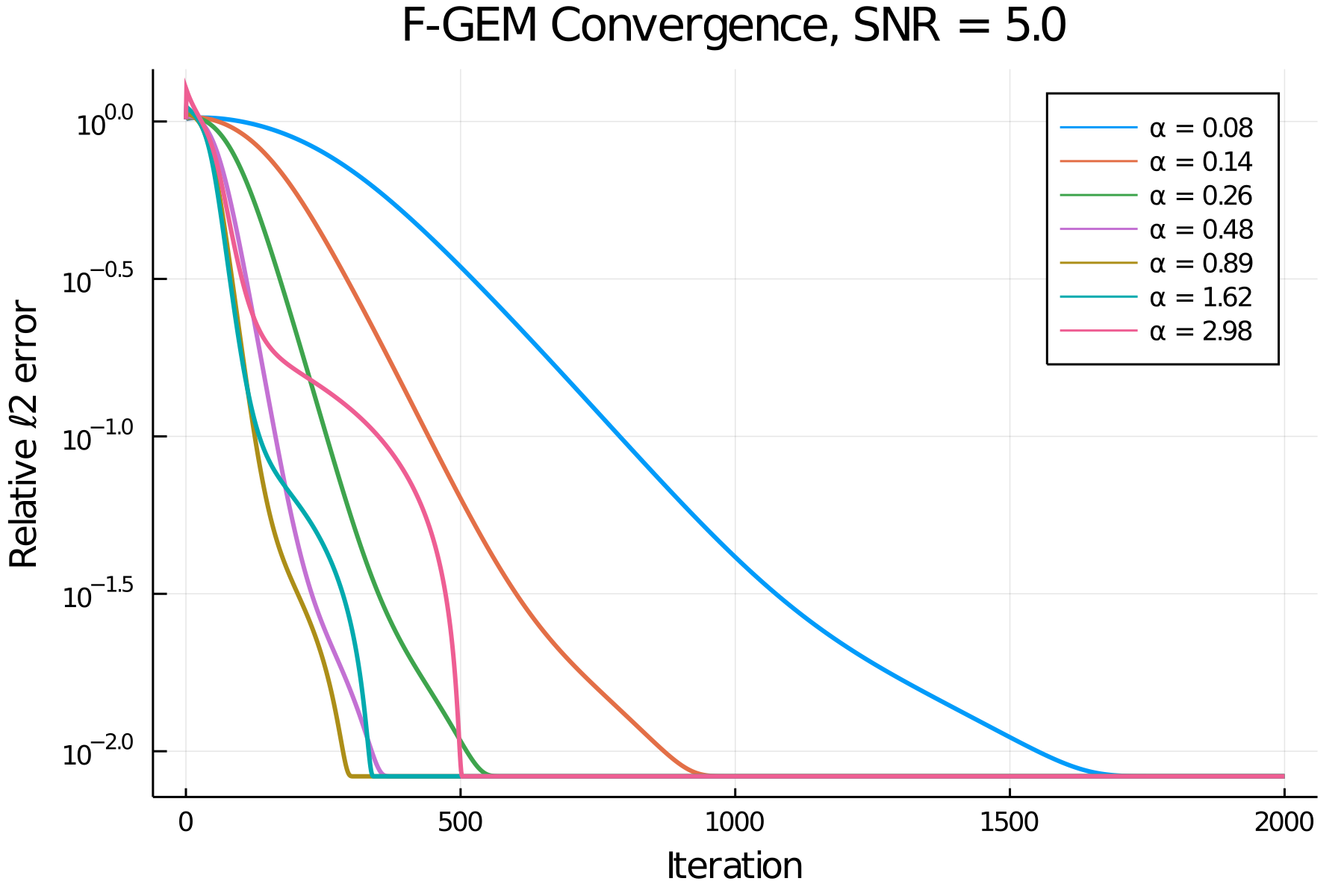}
    \includegraphics[width=0.32\linewidth]{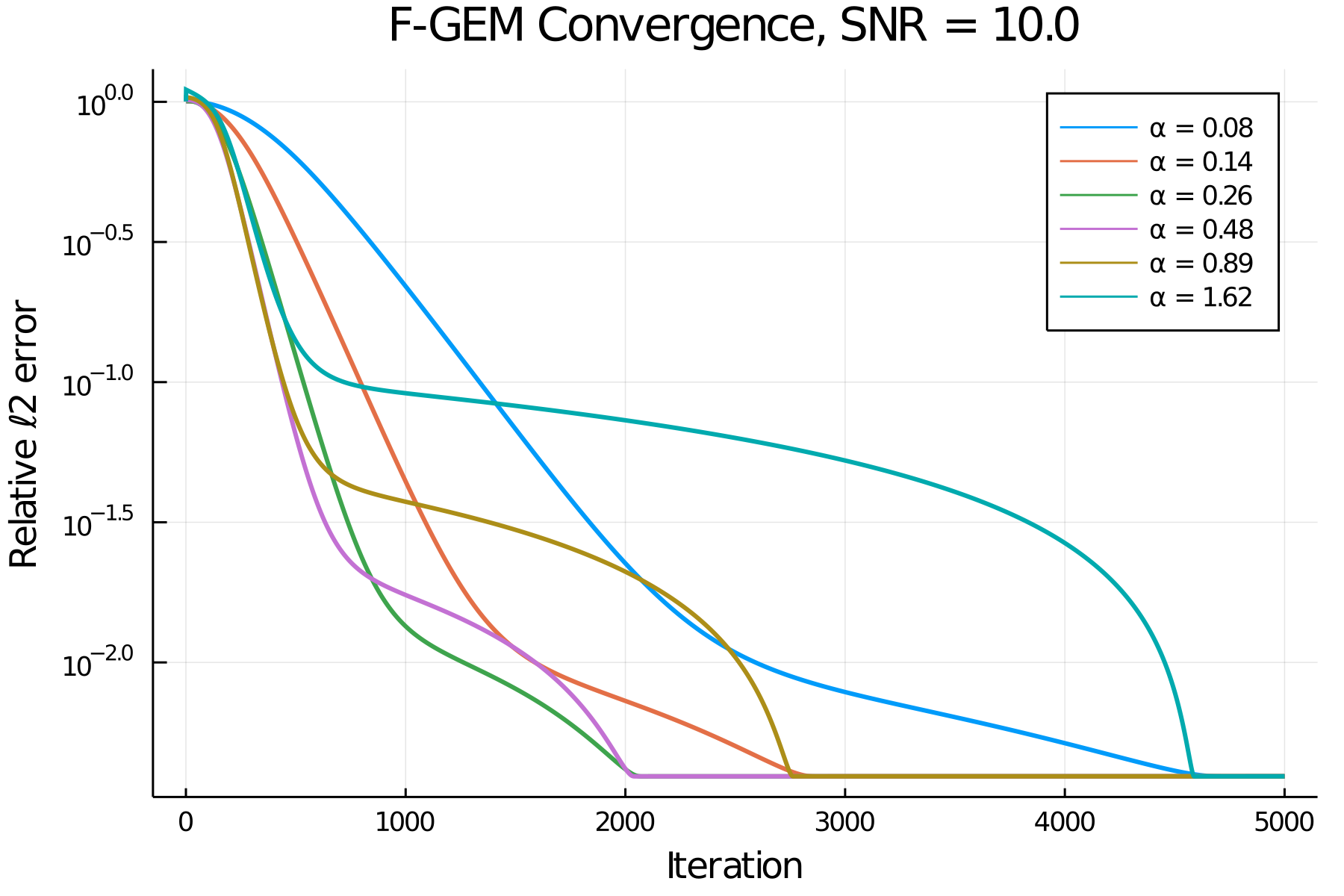}
    \includegraphics[width=0.32\linewidth]{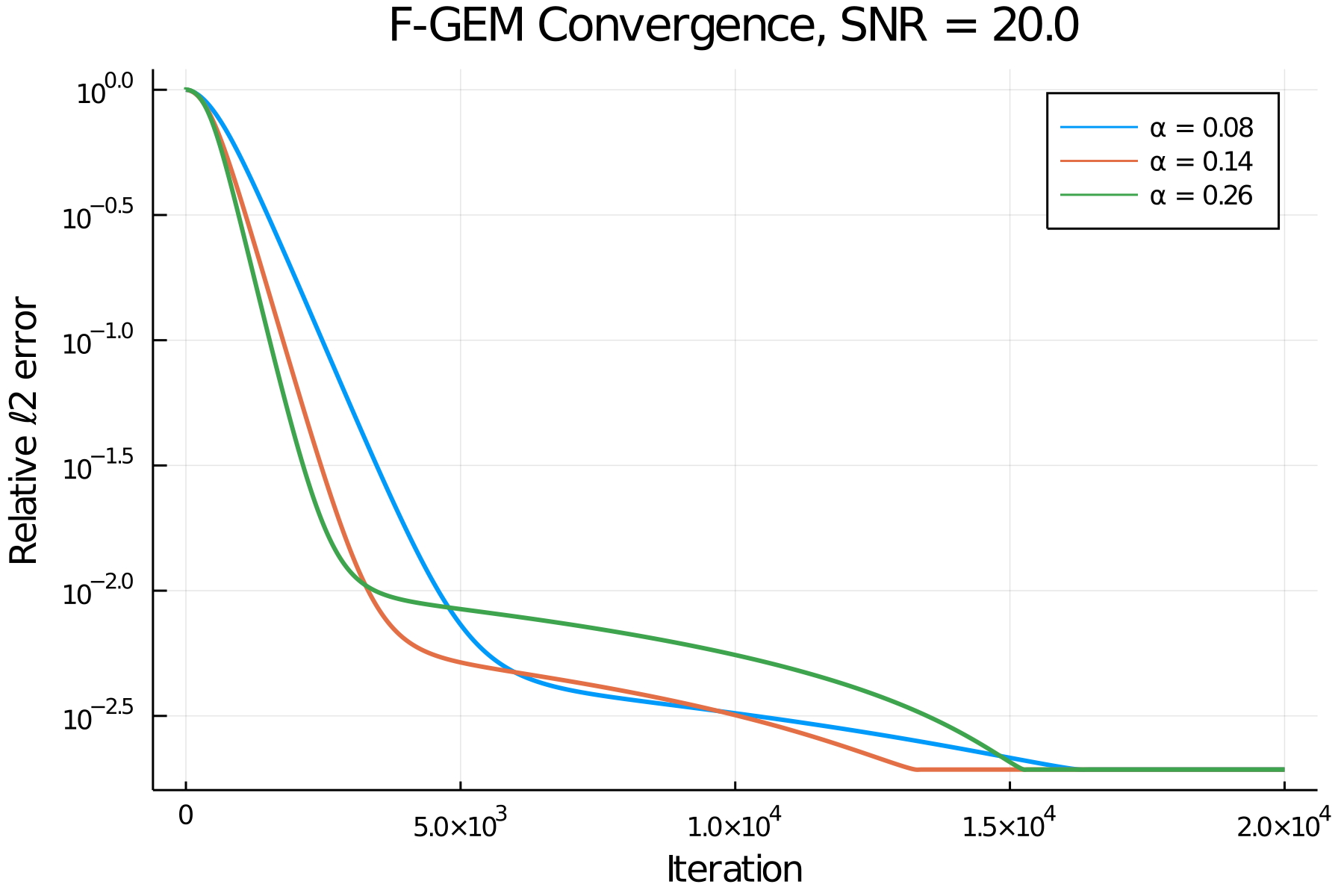}
    \includegraphics[width=0.32\linewidth]{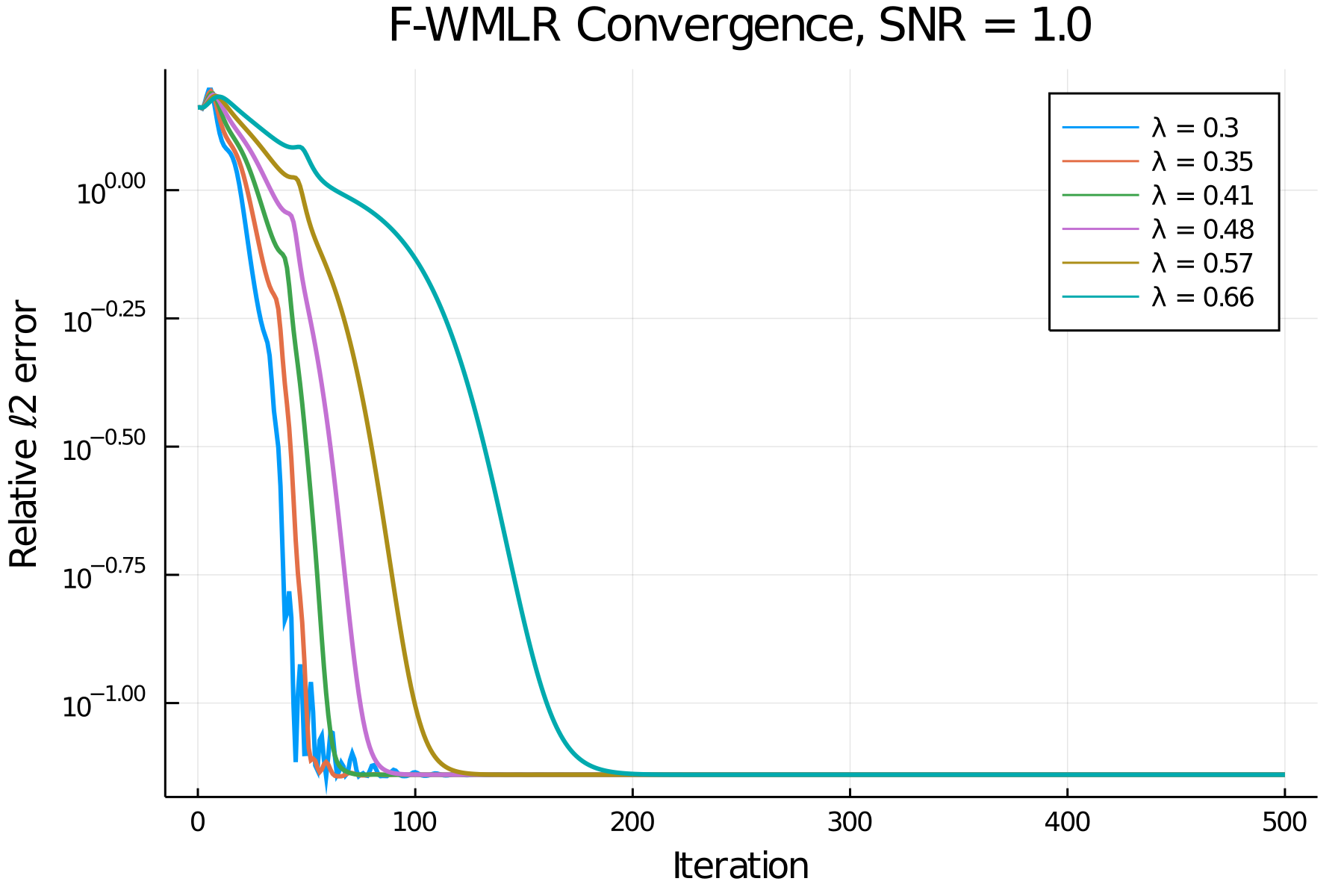}
    \includegraphics[width=0.32\linewidth]{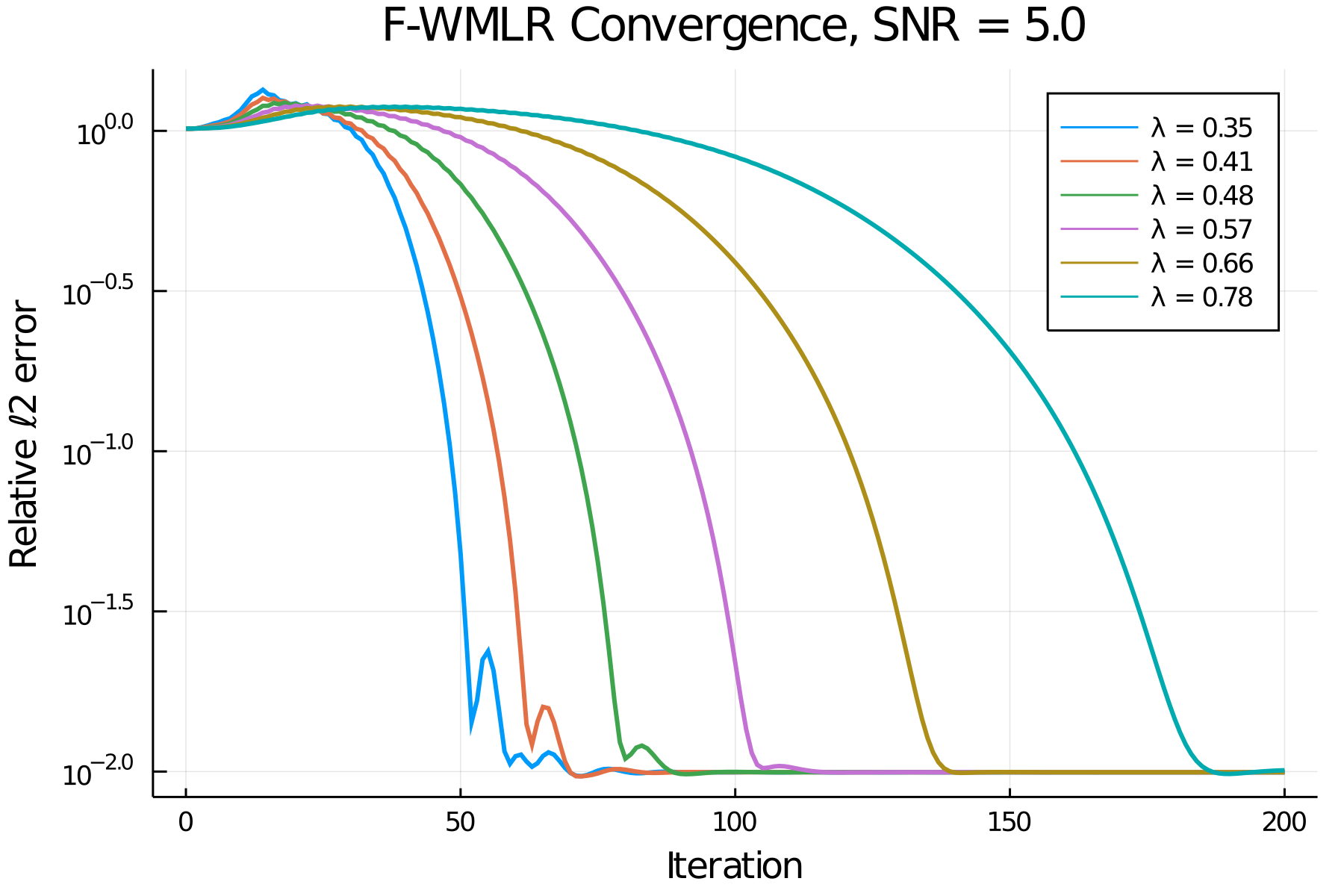}
    \includegraphics[width=0.32\linewidth]{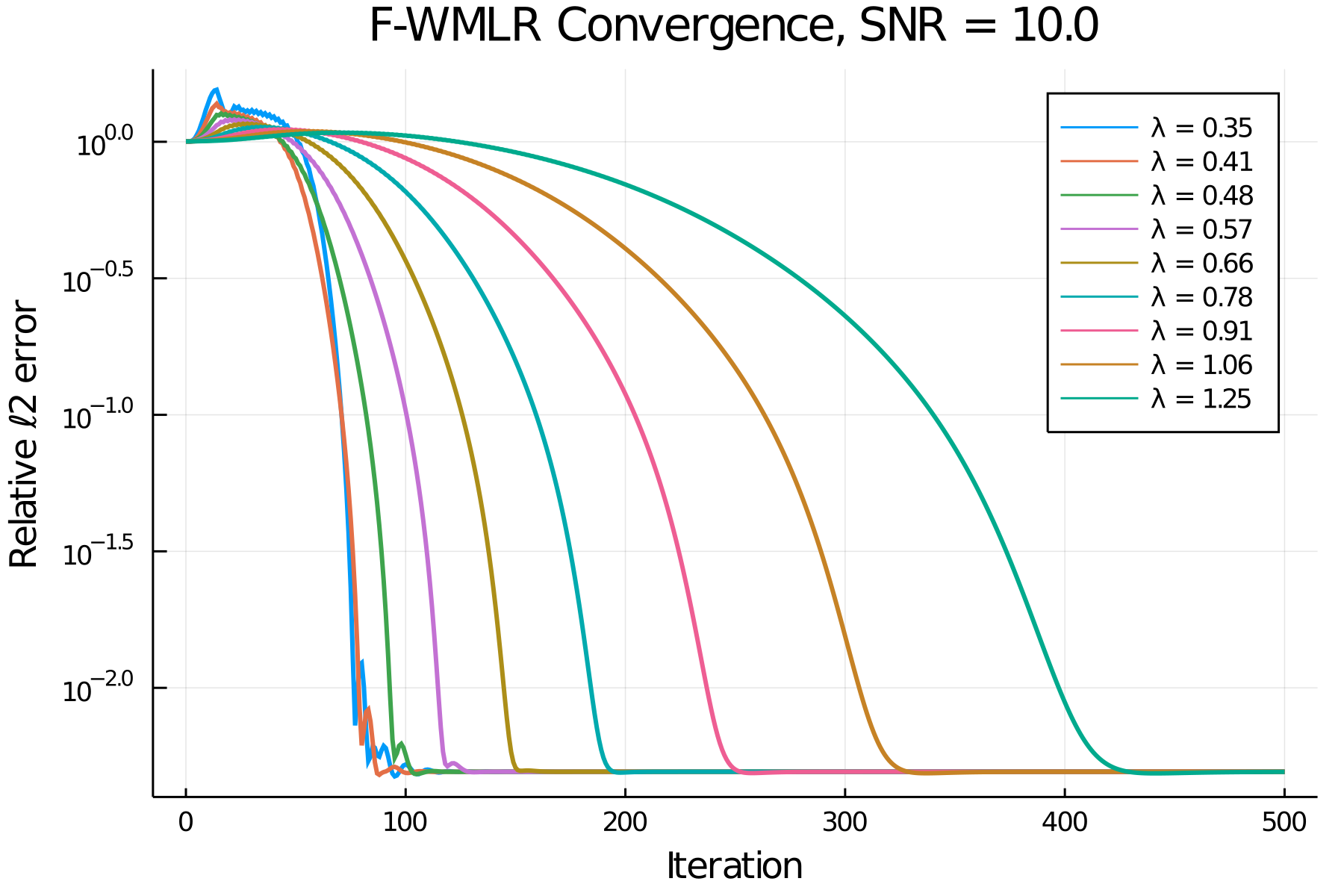}
    \includegraphics[width=0.32\linewidth]{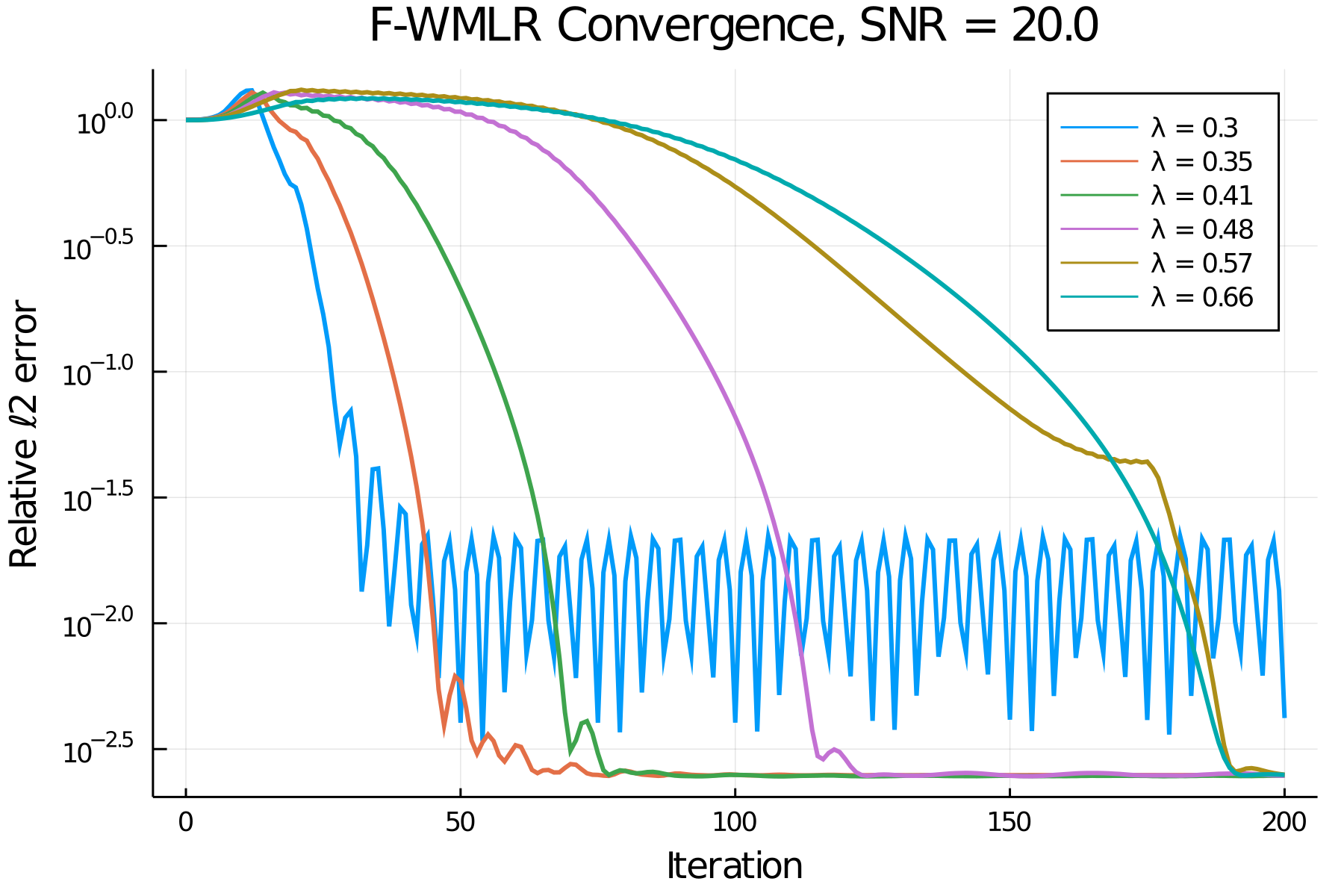}
    \caption{Convergence for different hyperparameter values.}
    \label{fig:params}
\end{figure}

\end{document}